\documentclass{article}

\usepackage{times,fullpage}
\usepackage{graphicx} 
\usepackage{subfigure} 

\usepackage{amsmath,amsthm,amssymb,pdfsync}  
\newtheorem*{rep@theorem}{\rep@title}
\newenvironment{oneshot}[1]{\def\rep@title{#1} \begin{rep@theorem}}{\end{rep@theorem}}

\usepackage{natbib}
\renewcommand{\cite}{\citep}

\usepackage{algorithm}
\usepackage{algorithmic}

\renewcommand{\algorithmicrequire}{\textbf{Input:}}
\renewcommand{\algorithmicensure}{\textbf{Output:}}

\usepackage{custom_macro} 

\begin{document}

%
%
%
%
%
%

\title{Recovery guarantee of weighted low-rank approximation via alternating minimization}

\author{
Yuanzhi Li\thanks{Department of Computer Science, Princeton University. Email:\texttt{yuanzhil@cs.princeton.edu}}, ~
Yingyu Liang\thanks{Department of Computer Science, Princeton University. Email: \texttt{yingyul@cs.princeton.edu}},~
Andrej Risteski\thanks{Department of Computer Science, Princeton University. Email:\texttt{risteski@cs.princeton.edu}}
}
\date{}

\maketitle

\begin{abstract} 
Many applications require recovering a ground truth low-rank matrix from noisy observations of the entries, which in practice is typically formulated as a weighted low-rank approximation problem and solved by non-convex optimization heuristics such as alternating minimization.
In this paper, we provide provable recovery guarantee of weighted low-rank via a simple alternating minimization algorithm. In particular, for a natural class of matrices and weights and without any assumption on the noise, we bound the spectral norm of the difference between the recovered matrix and the ground truth, by the spectral norm of the weighted noise plus an additive error that decreases exponentially with the number of rounds of alternating minimization, from either initialization by SVD or, more importantly, random initialization. 
These provide the first theoretical results for weighted low-rank via alternating minimization with non-binary deterministic weights, significantly generalizing those for matrix completion, the special case with binary weights, since our assumptions are similar or weaker than those made in existing works.  Furthermore, this is achieved by a very simple algorithm that improves the vanilla alternating minimization with a simple clipping step.

The key technical challenge is that under non-binary deterministic weights, na\"ive alternating steps will destroy the incoherence and spectral properties of the intermediate solutions, which are needed for making progress towards the ground truth.
We show that the properties only need to hold in an average sense and can be achieved by the clipping step. 

We further provide an alternating algorithm that uses a whitening step that keeps the properties via SDP and Rademacher rounding and thus requires weaker assumptions. This technique can potentially be applied in some other applications and is of independent interest.
\end{abstract} 

\section{Introduction}


Recovery of low-rank matrices has been a recurring theme in recent years in machine learning, signal processing, and numerical linear algebra, since in many applications, the data is a noisy observation of a low-rank ground truth matrix. 
Typically, the noise on different entries is not identically distributed, which naturally leads to a weighted low-rank approximation problem: given the noisy observation $\bM$, one tries to recover the ground truth by finding $\btM$ that minimizes $\|\bM - \btM\|^2_\bW = \sum_{ij} \bW_{i,j} (\bM_{i,j} - \btM_{i,j})^2$ where the weight matrix $\bW$ is chosen according to prior knowledge about the noise.
For example, the co-occurrence matrix for words in natural language processing applications~\cite{pennington2014glove,arora2015random} is such that the noise is larger when the co-occurrence of two words is rarer. When doing low-rank approximation on the co-occurrence matrix to get word embeddings, it has been observed empirically that a simple weighting can lead to much better performance than the unweighted formulation (see, e.g.,~\cite{levy2014neural}). In  biology applications, it is often the case that the variance of the noise is different for each entry of a data matrix,  due to various reasons such as different properties of different measuring devices. A natural approach to recover the ground truth matrix is to solve a weighted low-rank approximation problem where the weights are inversely proportional to the variance in the entries~\cite{gadian1982nuclear,wentzell1997maximum}. Even for collaborative filtering, which is typically modeled as a matrix completion problem that assigns weight $1$ on sampled entries and $0$ on non-sampled entries,  one can achieve better results when allowing non-binary weights~\cite{srebro2003weighted}. 

In practice, the weighted low-rank approximation is typically solved by non-convex optimization heuristics. One of the most frequently used is alternating minimization, which sets $\btM$ to be the product of two low-rank matrices and alternates between updating the two matrices. Although it is a natural heuristic to employ and also an interesting theoretical question to study, to the best of our knowledge there is no guarantee for alternating minimization for weighted low-rank approximation. Moreover, general weighted low-rank approximation is NP-hard, even when the ground truth is a rank-1 matrix~\cite{gillis2011low}. 

A special case of weighted low-rank approximation is matrix completion, where the weights are binary.
Most methods proposed for solving this problem rely on the assumptions that the observed entries are sampled uniformly at random, and additionally often the observations need to be re-sampled across different iterations of the algorithm. This is inherently infeasible for the more general weighted low-rank approximation, and thus their analysis is not portable to the more general problem. The few exceptions that work with deterministic weights are~\cite{heiman2014deterministic,lee2013matrix,bhojanapalli2014universal}. In this line of work the state-of-the-art is~\cite{bhojanapalli2014universal}, who proved recovery guarantees under the assumptions that the ground truth has a strong version of incoherence and the weight matrix has a sufficiently large spectral gap. However, their results still only work for binary weights, use a nuclear norm convex relaxation and do not consider noise on the observed entries. 

In this paper, we provide the first theoretical guarantee for weighted low-rank approximation via alternating minimization, under assumptions generalizing those in~\cite{bhojanapalli2014universal}. In particular, assuming that the ground truth has a strong version of incoherence and the weight matrix has a sufficiently large spectral gap, we show that the spectral norm of the difference between the recovered matrix and the ground truth matrix is bounded by the spectral norm of the weighted noise plus an additive error term that decreases exponentially with the number of rounds of alternating minimization, from either initialization by SVD or, more importantly, random initialization.  We emphasize that the bounds hold without any assumption on the noise, which is particularly important for handling complicated noise models.
Since uniform sampling can satisfy our assumptions, our guarantee naturally generalizes those in previous works on matrix completion. See Section~\ref{sec:comp} for a detailed comparison.

The guarantee is proved by showing that the distance between the intermediate solution and the ground truth is improved at each iteration, which in spirit is similar to the framework in previous works. However, the lack of randomness in the weights and the exclusion of re-sampling (i.e., using independent samples at each iteration) lead to several technical obstacles that need to be addressed. Our proof of the improvement is then significantly different (and more general) from previous ones. 
In particular, showing improvement after each step is only possible when the intermediate solution has some additional special properties in terms of incoherence and spectrum. Prior works ensure such properties by using re-sampling (and sometimes assumptions about the noise), which are not available in our setting. 
We address this by showing that the spectral property only needs to hold in an average sense, which can be achieved by a simple clipping step. This results in a very simple algorithm that almost matches the practical heuristics, and thus provides explanation for them and also suggests potential improvement of the heuristics.


\paragraph{Further results} 
The above results build on the insight that the spectral property only need to hold in an average sense. However, we can even make sure that the spectral property holds at each step strictly by a whitening step. More precisely, the clipping step is replaced by a whitening step using SDP and Rademacher rounding, which ensures that the intermediate solutions are incoherent and have the desired spectral property (the smallest eigenvalues of some related matrices are bounded).  The technique of maintaining the smallest eigenvalues may be applicable to some other non-convex problems, and thus is of independent interest. The details are presented in Appendix~\ref{app:proof_sdp}.

Furthermore, combining our insight that the spectral property only need to hold in an average sense with the framework in~\cite{sun2015Guaranteed}, one can show provable guarantees for the family of algorithms analyzed there, including stochastic gradient descent. 
We will demonstrate this by including the proof details for stochastic gradient descent in a future version.

\section{Related work} \label{sec:relatedwork}


Being a common practical problem (e.g.,~\cite{lu1997weighted,srebro2003weighted,li2010improving,eriksson2012efficient}), multiple heuristics for non-convex optimization such as alternating minimization have been developed, but they come with no guarantees.  On the other hand, weighted low-rank approximation is NP-hard in the worst case, even when the ground truth is a rank-1 matrix~\cite{gillis2011low}. 

On the theoretical side, the only result we know of is \cite{iiya2016weighted}, who provide a fixed-parameter tractability result when \emph{additionally} the weight matrix is low-rank. Namely, when the weight matrix has rank $r$, they provide an algorithm for outputting a matrix $\btM$ which approximates  the optimization objective 
up to a $1+\epsilon$ multiplicative factor, and runs in time $n^{O(k^2 r/\epsilon)}$.  

A special case of weighted low rank approximation is matrix completion, where the goal is to recover a low-rank matrix from a subset of  the matrix entries and corresponds to the case when the weights are in $\{0,1\}$. For this special case much more is known theoretically. 
It is known that matrix completion is NP-hard in the case when the $k=3$~\cite{peeters1996orthogonal}.  
Assuming that the matrix is incoherent and the observed entries are chosen uniformly at random, \citet{candes2009exact} showed that nuclear norm convex relation can recover an $n \times n$ rank-$k$ matrix using $m = O(n^{1.2}k \log(n))$ entries. 
The sample size is improved to $O(nk\log(n))$ in subsequent papers~\cite{candes2010power,recht2011simpler,gross2011recovering}.
\citet{candes2010matrix} relaxed the assumption to tolerate noise and showed the nuclear norm convex relaxation can lead to a solution such that  the Frobenius norm of the error matrix is bounded by $O(\sqrt{n^3/m})$ times that of the noise matrix.
However, all these results are for the restricted case with uniformly random binary weight matrices.

The only relaxations to random sampling to the best of our knowledge are in~\cite{heiman2014deterministic,lee2013matrix,bhojanapalli2014universal}. In this line the state-of-the-art is~\cite{bhojanapalli2014universal}, where the support of the observation is a $d$-regular expander such that the weight matrix has a sufficiently large spectral gap.
However, it only works for binary weights, and is for a nuclear norm convex relaxation and does not incorporate noise. 

Recently, there is an increasing interest in analyzing non-convex optimization techniques for matrix completion. 
In two seminal papers~\cite{jain2013low,hardt2014understanding}, it was shown that with an appropriate SVD-based initialization, the alternating minimization algorithm (with a few modifications) recovers the ground-truth. These results are for random binary weight matrix and crucially rely on re-sampling (i.e., using independent samples at each iteration), which is inherently not possible for the setting studied in this paper. More recently, \citet{sun2015Guaranteed} proved recovery guarantees for a family of algorithms including alternating minimization on matrix completion without re-sampling. However, the result is still for random binary weights and has not considered noise. 
More detailed comparison of our result with prior work can be found in Section~\ref{sec:algorithm_main}, and comments on whether their arguments can be applied in our setting can be found in Section~\ref{sec:lemma_main}.

We also mention \cite{negahban2012restricted} who consider random sampling, but one that is not uniformly random across the entries. In particular, their sampling produces a rank-1 matrix. (Additionally, they require the ground truth matrix to have nice properties such as low-rankness and spikiness.) The rank-1 assumption on the weight matrix is typically not true for many applications that introduce the weights to battle the different noise across the different entries of the matrix.

Finally, two related works are \cite{DBLP:journals/corr/Bhojanapalli0S14,bhojanapalli2015drop}. The former implements faster SVD decomposition via weighted low rank approximation. However, here the weights in the weighted low rank problem come from leverage scores, so have a very specific structure, specially designed for performing SVD decompositions. The latter
concerns optimization of strongly convex functions $f(\bV)$ when $\bV$ is in the set of positive-definite matrices. It does this in a non-convex manner, by setting $\bV = \bU \bU^{\top}$ and using the entries of $\bU$ as variables. Our work focus on the recovery of the ground truth under the generative model, rather than on the optimization.

\section{Problem definition and assumptions} \label{sec:problem}

For a matrix $\bA$, let $\bA_i$ denote its $i$-th column, $\bA^j$ denote its $j$-th row, and $\bA_{i,j}$ denote the element in $i$-th row and $j$-th column. Let $\hp$ denote the Hadamard product, i.e., $\bC = \bA \hp \bB$ means $\bC_{i,j} = \bA_{i,j} \bB_{i,j}$.

Let $\bMg \in \Real^{n\times n}$  be a rank-$k$ matrix. 
Given the observation $\bM = \bMg + \bNoise$ where $\bNoise$ is a noise matrix, we want to recover the ground truth $\bMg$ by solving the weighted low-rank approximation problem for $\bM$ and a non-negative weight matrix $\bW$:
$$\min_{\btM \in \mathcal{R}_k} \nbr{\btM - \bM}^2_\bW $$ 
where $\mathcal{R}_k$ is the set of rank-$k$ $n$ by $n$ matrices, and $\nbr{\bA}^2_\bW = \sum_{i,j} \bW_{i,j} \bA_{i,j}^2$ is the weighted Frobenius norm.  
Our goal is to specify conditions about $\bMg$ and $\bW$, under which $\bMg$ can be recovered up to small error by alternating minimization, i.e., set $\btM = \bX \bY^\top$ where $\bX$ and $\bY$ are $n$ by $k$ matrices, and then alternate between updating the two matrices. Ideally, the recovery error should be bounded by $\|\bW \hp \bNoise \|_2$, since this allows selecting weights according to the noise to make the error bound small.

As mentioned before, the problem is NP-hard in general, so we will need to impose some conditions.  We summarize our assumptions as follows, and then discuss their necessity and the connections to existing ones.
\begin{itemize} 
\item[$\bold{(A1)}$] \emph{Ground truth is incoherent}: $\bMg$ has SVD $\bU \bSigma \bV^{\top}$, where 
$\max_{i = 1}^n \{ ||\bU^i||_2^2, ||\bV^i||_2^2 \} \le \frac{\mu k}{n}.$ Additionally, assume $\sigma_{\max}(\bSigma) = \Theta(1)$. (See discussion below.) 
Denote its condition number as $\ckappa = \sigma_{\max}(\bSigma) /\sigma_{\min}(\bSigma)$.
\item[$\bold{ (A2)}$] \emph{Weight matrix has a spectral gap}: $||\bW - \bE||_2 \le \gamma n$, where $\gamma < 1$ and $\bE$ is the all-one matrix.
\item[$\bold{ (A3)}$] \emph{Weight is not degenerate}: Let $\bD_i = \diag(\bW^i)$, i.e., $\bD_i$ is a diagonal matrix whose diagonal entries are the $i$-th row of $\bW$.
Then there are $0 < \lambdal \le 1 \le \lambdau$:
$$\hspace{-6mm}\lambdal \bI \preceq \bU^{\top} \bD_i \bU \preceq \lambdau \bI, \text{and}~\lambdal \bI \preceq \bV^{\top} \bD_i \bV \preceq \lambdau \bI (\forall i \in [n]). $$
\end{itemize}

The incoherence assumption on the ground truth matrix is standard in the context of matrix completion. It is known that this is necessarily required for recovering the ground truth matrix. The assumption that $\sigma_{\max}(\bSigma) = \Theta(1)$ is without loss of generality: one can estimate $\sigma_{\max}(\bSigma)$ up to a constant factor, scale the data and apply our results. The full details are included in the appendix. 

The spectrum assumption on the weight matrix is a natural generalization of the randomness assumption typically made in matrix completion scenario (e.g.,~\cite{candes2010matrix,jain2013low,hardt2014understanding}). In that case, $\bW$ is a matrix with $d = \Omega(\log n)$-nonzeros in each row chosen uniformly at random, which corresponds to $\gamma = O\left(\frac{1}{\sqrt{d}} \right)$ in $\bold{(A2)}$.
Our assumption is also a generalization of the one in~\cite{bhojanapalli2014universal}, which requires $\bW$ to be $d$-regular expander-like (i.e., to have a spectral gap)  but is concerned only with matrix completion where the entries of $\bW$ can be 0 or 1 only. 

The final assumption $\bold{(A3)}$ is a generalization of the assumption A2 in~\cite{bhojanapalli2014universal} that, intuitively, requires the singular vectors to satisfy RIP (restricted isometry property). This is because when the weights are binary, $\bU^{\top} \bD_i \bU = \sum_{j \in S} (\bU^k)(\bU^k)^{\top}$ where $S$ is the support of $\bW^i$, so after proper scaling the assumption is a strict weakening of theirs.  They viewed it as a stronger version of incoherence, discussed the necessity and showed that it is implied by the strong incoherence property assumed in~\cite{candes2010power}.  
In the context of more general weights, the necessity of $\bold{(A3)}$ is even more clear, as elaborated below. 

Note that since $\bold{(A2)}$ does not require $\bW$ to be random or $d$-regular, it does not a-priori exclude the degenerate case that $\bW$ has one all-zero column. In that case, clearly one cannot hope to recover the corresponding column of $\bMg$. So, we need to make a third, non-degeneracy assumption about $\bW$, saying that it is ``correlated'' with $\bMg$. The assumption is actually quite weak in the sense that when $\bW$ is chosen uniformly at random, this assumption is true automatically: in those cases, $\E[\bD_i] = \bI$ and thus $\E[ \bU^{\top} \bD_i \bU] = \bI$ since $ \bU$ is orthogonal. A standard matrix concentration bound can then show that our assumption $\bold{(A3)}$ holds with high probability. Therefore, it is only needed when considering a deterministic $\bW$. Intuitively, this means that the weights should cover the singular vectors of $\bMg$. This prevents the aforementioned degenerate case when $\bW_i = 0$ for some $i$, and also some other degenerate cases. For example, consider the case when $\bNoise = 0$, all rows of $\bMg$ are the same vector with first $\Theta(\log n)$ entries being zero and the rest being one, and in one row of $\bMg$  the non-zeros entries all have zero weight. In this case, there is also no hope to recover $\bMg$, which should be excluded by our assumption.


\section{Algorithm and results} \label{sec:algorithm_main}

\begin{algorithm}[!t]
\caption{Main Algorithm ($\AlgMain$)}\label{alg:main}
\begin{algorithmic}[1]
\REQUIRE Noisy observation $\bM$, weight matrix $\bW$,   number of iterations $T$
\STATE Initialize $\bY_1$ using either $\bY_1 = \AlgInitial(\bM, \bW)$ or $\bY_1 = \AlgRInitial$
\FOR{$t = 1, 2, ..., T$}
\STATE $\btX_{t + 1} = \argmin_{\bX \in \Real^{n \times k}} \leftll \bM - \bX \bY^{\top}_t \rightll_{\bW}$
\STATE $\bbX_{t + 1} = \AlgClip(\btX_{t+1})$
\STATE $\bX_{t+ 1} = \QR(\bbX_{t + 1} )$
\STATE $\btY_{t + 1} = \argmin_{\bY \in \Real^{n \times k}} \leftll \bM - \bX_{t + 1} \bY^{\top} \rightll_{\bW}$
\STATE $\bbY_{t + 1} = \AlgClip(\btY_{t+1})$
\STATE $\bY_{t + 1} = \QR(\bbY_{t + 1} )$
\ENDFOR
\ENSURE $\btM = \bbX_{ T + 1} \bY_{T}$ 
\end{algorithmic}
\end{algorithm}

\begin{algorithm}[!t]
\caption{Clipping (\AlgClip)}\label{alg:clip}
\begin{algorithmic}[1]
\REQUIRE matrix $\btX$
\ENSURE matrix $\bbX$ with
$$\bbX^i =  \left\{ \begin{array}{ll}
         \btX^i & \mbox{if $\| \btX^i \|_2^2 \le \wt := \frac{2\mu k}{n}$}\\
        0 & \mbox{otherwise}.\end{array} \right.$$
\end{algorithmic}
\vspace{-2mm}
\end{algorithm}

We prove guarantees for the vanilla alternating minimization with a simple clipping step, from either SVD initialization or random initialization. 
The algorithm is specified in Algorithm~\ref{alg:main}. Overall, it follows the usual alternating minimization framework: it keeps two working matrices $\bX$ and $\bY$, and alternates between updating them. In an $\bX$ update step, it first updates $\bX$ to be the minimizer of the weighted low rank objective while fixing $\bY$, which can be done efficiently since now the optimization is convex. 
Then it performs a ``clipping" step which zeros out rows of the matrix with too large norm,\footnote{The clipping step zeros out rows with square $\ell_2$ norm twice larger than the upper bound $\mu k/n$ imposed by our incoherence assumption $\bold{(A1)}$. One can choose the threshold to be $c\mu k/n$ where $c\ge 2$ is a constant and can choose to shrink the row to have norm no greater than $\mu k/n$, and our analysis still holds. The current choices are only for ease of presentation.} and then make it orthogonal by QR-factorization.\footnote{The QR-factorization step is not necessary for our analysis. But since it is widely used in practice for numerical stability, we prefer to analyze the algorithm with QR.} 
At the end, the algorithm computes a final solution $\btM$ from the two iterates. 

The two iterates can be initialized by performing SVD on the weighted observation (Algorithm~\ref{alg:initial}), which is a weighted version of SVD initialization typically used in matrix completion. Moreover, we show that the algorithm works with random initialization (Algorithm~\ref{alg:rinitial}), which is a simple and widely used heuristic in practice but rarely understood well.

\begin{algorithm}[!t]
\caption{SVD Initialization (\AlgInitial)}\label{alg:initial}
\begin{algorithmic}[1]
\REQUIRE observation $\bM$, weight $\bW$ 
\STATE $(\btX, \bSigma, \btY) = \textnormal{rank-}k ~\SVD(\bW \hp \bM)$, i.e., the columns of $\btY$ are the top $k$ right singular vectors of $\bW \hp \bM$
\STATE $\bbY = \AlgClip(\btY)$, $\bY = \QR(\bbY )$
\ENSURE $\bY$
\end{algorithmic}
\end{algorithm}

\begin{algorithm}[!t]
\caption{Random Initialization (\AlgRInitial)}\label{alg:rinitial}
\begin{algorithmic}[1]
 \STATE Let $\bY \in \mathbb{R}^{n \times k}$ generated as $\bY_{i,j} = b_{i,j} \frac{1}{\sqrt{n}}$, where $b_{i,j}$'s are independent uniform from $\{-1,1\}$
\ENSURE $\bY$
\end{algorithmic}
\end{algorithm}

We are now ready to state our main results. 
Theorem~\ref{thm:alt_min} describes our guarantee for the algorithm with SVD initialization, and Theorem~\ref{thm:alt_min_rand} is for random initialization. 

\begin{thm}[Main, SVD initialization]  \label{thm:alt_min} 
Suppose $\bMg, \bW$ satisfy assumptions \textbf{(A1)-(A3)} with
\begin{align*}
\gamma  = O\left( \min\left\{ \sqrt{\frac{n}{D_1}}\frac{\lambdal  }{\ckappa \mu^{3/2} k^2} , 
\frac{\lambdal }{\ckappa^{3/2}\mu k^2}    \right\} \right),
\end{align*}
where $D_1 = \max_{i \in [n]} \|\bW^i\|_1$. 
Then after $O(\log(1/\epsilon))$ rounds of Algorithm~\ref{alg:main} with initialization from Algorithm~\ref{alg:initial} outputs a matrix $\btM$ that satisfies
\begin{align*}  
  \|\btM - \bMg\|_2 \le O\left(\frac{k \ckappa }{\lambdal}\right)\|\bW \hp \bNoise\|_2 + \epsilon.
\end{align*}
The running time is polynomial in $n$ and $\log(1/\epsilon)$.
\end{thm}

The theorem is stated in its full generality. 
To emphasize the dependence on the matrix size $n$, the rank $k$ and the incoherence $\mu$, we  can consider a specific range of parameter values where the other parameters (the spectral bounds, condition number, $D_1/n$) are constants. Also, these parameter values are typical in  matrix completion, which facilitates our comparison in the next subsection. 

\begin{cor}\label{cor:main} 
Suppose $\lambdal, \lambdau$ and  $\ckappa$ are all constants, $D_1 = \Theta(n)$, and $T = O(\log(1/\epsilon))$. Furthermore, 
\begin{align*}
\gamma  = O\left(  \frac{1  }{\mu^{3/2} k^2} \right).
\end{align*} 
Then Algorithm~\ref{alg:main} with initialization from Algorithm~\ref{alg:initial} outputs a matrix $\btM$ that satisfies
\begin{align*}  
  \|\btM - \bMg\|_2 \le O\left(k\right) \|\bW \hp \bNoise\|_2 + \epsilon.
\end{align*}
\end{cor}

\paragraph{Remarks}
The theorem bounds the spectral norm of the error matrix by the spectral norm of the weighted noise plus an additive error term that decreases exponentially with the number of rounds of alternating minimization. We emphasize that our guarantee holds for any $\bMg, \bW$ satisfying our deterministic assumptions; the high success probability is with respect to the execution of the algorithm, not to the input.  This ensures the freedom in choosing the weights to battle the noise. We also emphasize that the bounds hold \emph{without any assumption} on the noise, which is particularly important here since weighted low rank is typically applied to complicated noise models.

Bounding the error by $\|\bW \hp \bNoise\|_2$ is particularly useful when the noise is not uniform across the entries: prior knowledge about the noise (e.g., the different variances of noise on different entries) can be taken into account by setting up a reasonable weight matrix\footnote{Note that $\bW$ cannot be made arbitrarily small since it should satisfy our assumptions. Roughly speaking, $\bW$ has spectral norm $n$ and is flexible to take into account the prior knowledge about the noise. In particular, it can be set to the all one matrix, reducing to the unweighted case.}, such that $\|\bW \hp \bNoise\|_2$ can be significantly smaller than $\|\bNoise\|_2$. Also, in recovering the ground truth, a spectral norm bound is more preferred than a Frobenius norm bound, since typically the Frobenius norm is $\sqrt{n}$ larger than the spectral norm. 

Furthermore, when $\|\bW \hp \bNoise\|_2 = 0$ (as in matrix completion without noise), the ground truth is recovered in a geometric rate.

Finally, in matrix completion with uniform random sampled observations, the term $D_1$ concentrates around $n$, so $\frac{D_1}{n}$ disappears in this case.

\begin{thm}[Main, random initialization]  \label{thm:alt_min_rand} 
Suppose $\bMg, \bW$ satisfy assumptions \textbf{(A1)-(A3)} with
\begin{align*}
\gamma  &= O\left( \min\left\{ \sqrt{\frac{n}{D_1}}\frac{\lambdal }{ \ckappa\mu^{2} k^{5/2}} , 
\frac{\lambdal }{ \ckappa^{3/2} \mu^{3/2} k^{5/2}}    \right\} \right), 
\\ 
\|\bW\|_{\infty} & = O\left(\frac{\lambdal n}{k^2 \mu \log^2 n}\right), 
\end{align*}
where $D_1 = \max_{i \in [n]} \|\bW^i\|_1$. 
Then after $O(\log(1/\epsilon))$ rounds Algorithm~\ref{alg:main}  with initialization from Algorithm~\ref{alg:initial}  outputs a matrix $\btM$ that with probability at least $1-\frac{1}{n^2}$ satisfies
\begin{align*}  
  \|\btM - \bMg\|_2 \le O\left(\frac{k\ckappa}{\lambdal}\right) \|\bW \hp \bNoise\|_2 + \epsilon.
\end{align*}
The running time is polynomial in $n$ and $\log(1/\epsilon)$.
\end{thm}

\paragraph{Remarks} 
Compared to SVD initialization, we need slightly stronger assumptions for random initialization to work. There is an extra $1/(\mu^{1/2} k^{1/2})$ in the requirement of the spectral parameter $\gamma$. We note that the same error bound is obtained when using random initialization. Roughly speaking, this is because our analysis shows that the updates can make improvement under rather weak requirements that random initialization can satisfy, and after the first step the rest updates make the same progress as in the case using SVD initialization.

\subsection{Comparison with prior work}  \label{sec:comp}

\begin{table*}[t]
	\centering
		\begin{tabular}{c|c |c |c | c |c  | c}
		\hline
			   & weight  & determin.  &  tolerate &  alter. & order of $\gamma$  & \\
				& values &  weights &  noise & min. & (spectral gap) &   Bound on $ \Delta = \btM - \bMg$ \\
			\hline\hline
			(1) & 0-1 & no & yes & no & $ \frac{1}{\mu k^{1/2}\poly(\log n) }$ & $\|\Delta\|_F = O(\sqrt{\frac{n^3}{m} }\|\bNoise_\Omega\|_F)$ \\			
			\hline
			(2) & 0-1 & no & yes & no  & $\sqrt{\frac{1}{\mu k \log n}} $  & $\|\Delta\|_F = O( \frac{n^2 \sqrt{k}}{m} \|\bNoise_\Omega\|_2)$ \\
			\hline
		(3)	& 0-1 & yes & no & no  & $\frac{1}{\mu k} $ & exact recovery \\					
			\hline
			(4) & 0-1 & no & yes & yes & $\frac{1}{k \epsilon \sqrt{\mu  \log n}}$  & $\|\Delta\|_F \le \epsilon \|\bMg + \bNoise \|_F $ \\
			\hline
			(5) & 0-1 & no & no & yes & $\frac{1}{ \max\{\sqrt{k\mu \log n}, \mu k^{3.5}\}} $  &  exact recovery \\
			\hline
			ours (SVD init) & real & yes & yes & yes & $\frac{1}{\mu^{3/2} k^2 }$  &   $\|\Delta\|_2  = O\left(k  \right) \|\bW \hp \bNoise\|_2 + \epsilon $ \\
			\hline
			ours (random init) & real & yes & yes & yes & $\frac{1}{\mu^2 k^{5/2}}$  &   $\|\Delta\|_2  = O\left(k  \right) \|\bW \hp \bNoise\|_2 + \epsilon $ \\
					\hline
		\end{tabular}
\caption{Comparison with related work on matrix completion: (1) \cite{candes2010matrix}; (2) \cite{keshavan2009matrix}; (3) \cite{bhojanapalli2014universal}; (4) \cite{hardt2014understanding}. (5) \cite{sun2015Guaranteed}. Technical details are ignored. Especially, parameters other than the matrix size $n$, the rank $k$ and the incoherence $\mu$ are regarded as constants.  
	}
	\label{tab:com}
\end{table*}

For the sake of completeness, we will  give a more detailed comparison with representative prior work on matrix completion from Section~\ref{sec:relatedwork}, emphasizing the dependence on $n, k$ and $\mu$ and regarding the other parameters as constants.
We first note that when the $m$ observed entries are sampled uniformly at random from an $n$ by $n$ matrix, the corresponding binary weight matrix will have a spectral gap $\gamma = O(\sqrt{\frac{n}{m} })$ (see, e.g.,~\cite{feige2005spectral}). 
Converting the sample bounds in the prior work to the spectral gap, we see that in general our result has worse dependence on parameters like the rank than those by convex relaxations, but has slightly better dependence than those by alternating minimization. 
The comparison is summarized in Table~\ref{tab:com}.

The seminal paper~\cite{candes2009exact} showed that a nuclear norm convex relaxation approach can recover the ground truth matrix using $m = O(n^{1.2} k \log^2 n)$ entries chosen uniformly at random and without noise. The sample size was improved to $O(nk \log^6 n)$ in~\cite{candes2010power} and then $O(nk \log n)$ in subsequent papers. \citet{candes2010matrix} generalized the result to the case with noise: the same convex program using $m = O(nk \log^6 n)$ entries recovers a matrix $\btM$ s.t.
$\|\btM - \bMg\|_F \le (2+4 \sqrt{{(2 + p) n}/{p}})\|\bNoise_\Omega\|_F $
where $p = {m}/{n^2}$ and $\bNoise_\Omega$ is the noise projected on the observed entries. 

\citet{keshavan2009matrix} showed that with $ m = O(n  \mu k \log n)$, one can recover a matrix $\btM$ such that
$ \leftll  \bMg - \btM \rightll_F = O \left( \frac{n^2 \sqrt{k}}{m} \|\bNoise_\Omega\|_2\right)$ by an optimization over a Grassmanian manifold.

\citet{bhojanapalli2014universal} relaxed the assumption that the entries are randomly sampled. They showed that the nuclear norm relaxation recovers the ground truth, assuming that the support $\Omega$ of the observed matrix forms a $d$-regular expander graph (or alike), i.e., $|\Omega| = d n$, $\sigma_1(\Omega) = d$ and $\sigma_2(\Omega) \le c \sqrt{d}$ and $d \ge c^2 \mu^2 k^2$. This would correspond to a parameter $\gamma = O(\frac{1}{\mu k})$ for us. They did not consider the robustness to noise.

\citet{hardt2014understanding} showed that with an appropriate initialization alternating minimization recovers the ground truth approximately. Precisely, they assumed $\bNoise$ satisfies: (1). $\mu(\bNoise) \lesssim \sigma_{\min}(\bMg)^2$;(2). $\|\bNoise\|_{\infty} \le \frac{\mu}{n} \|\bMg\|_F$.
Then, he shows that $\log(\frac{n}{\epsilon} \log n)$ alternating minimization steps recover a matrix $\btM$ such that 
$\|\btM - \bMg\|_F \le \epsilon \| \bM\|_F$
provided that 
$pn \ge k (k + \log (n/\epsilon)) \mu \times \left( \frac{\|\bMg\|_F + \|\bNoise\|_F/\epsilon}{\sigma_k}\right)^2 \left( 1 - \frac{\sigma_{k+1}}{\sigma_k} \right)^5$
where $\sigma_k$ is the $k$-th singular value of the ground-truth matrix. The parameter $\gamma$ corresponding to the case considered there would be roughly $O(\frac{1}{k \sqrt{\mu \log n}})$. While their algorithm has a good tolerance to noise, $\bNoise$ is assumed to have special structure for him that we do not assume in our setting. 

\citet{sun2015Guaranteed} proved recovery guarantees for a family of algorithms including alternating minimization on matrix completion. 
They showed that by using $m = O(n k \max\{\mu \log n, \mu^2 k^6\})$ randomly sampled entries without noise, the ground truth can be recovered in a geometric rate. This corresponds to a spectral gap of $O\left(\frac{1}{ \max\{\sqrt{k\mu \log n}, \mu k^{3.5}\}} \right)$.
Our result is more general and also handles noise. When specialized to their setting, we also have a geometric rate with a slightly better dependence on the rank $k$ but a slightly worse dependence on the incoherence $\mu$.


\section{Proof sketch} \label{sec:lemma_main}

Before going into our analysis, we first discuss whether arguments in prior work can be applied.
Most of the work on matrix completion uses convex optimization and thus their analysis is not applicable in our setting.
There indeed exists some other work that analyzes non-convex optimization for matrix completion, and it is tempting to adopt their arguments. 
However, there exist fundamental difficulties in porting their arguments.
All of them crucially rely on the randomness in sampling the observed entries. \citet{keshavan2009matrix} analyzed optimization over a Grassmanian manifold, which uses the fact that $\E[\bW \hp \bS] = \bS$ for any matrix $\bS$. 
In~\cite{jain2013low,hardt2014understanding}, re-sampling of new observed entries in different iterations was used to get around the dependency of the iterates on the sample set, a common difficulty in analyzing alternating minimization. The subtlety and the drawback of re-sampling were discussed in detail in~\cite{bhojanapalli2014universal,candes2015phase,sun2015Guaranteed}.  We note that~\cite{sun2015Guaranteed} only needs sampling before the algorithm starts and does not need re-sampling in different iterations, but still relies on the randomness in the sampled entries.
In particular, in all the aforementioned work,  the randomness guarantees that the iterates $\bX, \bY$ stay incoherent and have good spectrum properties. Given these, alternating minimization can make progress towards the ground truth in each iteration. 
Nevertheless, since we focus on deterministic weights, such randomness is inherently infeasible in our setting.
In this case, after just one iteration, it is unclear if the iterates can have incoherence and good spectrum properties required to progress towards the ground truth, even under our current assumptions.
The whole algorithm thus breaks down.
To address this, we show that it is sufficient to ensure the spectral property in an average sense and then introduce our clipping step to achieve that, arriving at our current algorithm.

Here for simplicity, we drop the subscription $t$ in all iterates, and we only focus on important factors, dropping other factors and the big-$O$ notation. We only consider the case when $\bW \hp \bNoise = 0$, so as to emphasize the main technical challenges.

On a high level, our analysis of the algorithm maintains potential functions $\cds(\bX, \bU)$ and $\cds(\bY, \bV)$ between our working matrices $\bX, \bY$ and the ground truth $\bU, \bV$ (recall that $\bMg = \bU \bSigma \bV^{\top}$): 
$$
\cds(\bX, \bU) = \min_{\bQ \in \ortho_{k \times k}} \| \bX \bQ - \bU \|_2$$
and
$$ \cds(\bY, \bV) = \min_{\bQ \in \ortho_{k \times k}} \| \bY \bQ - \bV \|_2,
$$
where $\ortho_{k \times k}$ are the set of $k \times k$ rotation matrices.
The key is to show that they decrease after each update step, so $\bX$ and $\bY$ get closer to the ground truth.\footnote{Note that we also need a good initialization, which can be done by SVD.  Since our analysis requires rather weak warm start, we are able to show that simple random initialization is also sufficient (at the cost of slightly worse bounds).} The strategy of maintaining certain potential function measuring the distance between the iterates and the ground truth is also used in prior work~\cite{bhojanapalli2014universal,candes2015phase,sun2015Guaranteed}.
We will point out below the key technical difficulties that are not encountered in prior work and make our analysis substantially different.  The complete proofs are provided in the appendix due to space limitation.

\subsection{Update}
We would like to show that after an $\bX$ update, the new matrix $\btX$ satisfies  $\cds(\bX, \bU) \leq {\cds(\bY, \bV)}/{2} + c$ for some small $c$ (similarly for a $\bY$ update). 

Consider the update step 
$$\btX \leftarrow \argmin_{\bA \in \Real^{n \times k}} \leftll \bM - \bA \bY^{\top} \rightll_{\bW}.$$
By setting the gradient to 0 and with some algebraic manipulation, we have $\displaystyle \btX - \bU \bSigma \bV^{\top} \bY =  \bG$
where 
$$
\bG^i := \bU^i \bSigma \bV^{\top} \bYb \bYb^{\top} \bD_i  \bY (\bY^{\top} \bD_i \bY)^{-1}.
$$
where $\bD_i = \diag(\bW^i)$. Since $\btX$ is the value prior to performing QR decomposition, we want to show that $\btX$ is close to $\bU^i \bSigma \bV^{\top} \bY$, i.e., the error term $\bG$ on right hand side is small.
In the ideal case when the error term is 0, then $\btX = \bU \bSigma \bV^{\top} \bY$ and thus $\cds(\bX, \bU) = 0$, meaning that with one update $\btX$  already hits into the correct subspace. So we would like to show that it is small so that the iterate still makes progress.  
Let 
$$\bP_i = \bV^{\top} \bYb \bYb^{\top} \bD_i  \bY~~\textrm{and}~~\bO_i = (\bY^{\top} \bD_i \bY)^{-1},$$ 
so that $\bG^i =  \bU^i \bSigma \bP_i \bO_i$. Now the two challenges are to bound $\bP_i$  and $\bO_i$. 

Let us first consider the simpler case of matrix completion, where the entries of the matrix are randomly sampled by probability $p$. Then $\bD_i$ is a random diagonal matrix with $\E[\bD_i] = \bI$ and $\E[\bD_i^2] = \frac{1}{p} \bI$. 
Furthermore, for $n \times k$ orthogonal matrices $\bY$, $\bO_i = (\bY^{\top} \bD_i \bY)^{-1}$ concentrates around $\bI$. Then in expectation,  $||\bP_i||$ is about $||\bV^{\top} \bYb ||/\sqrt{p}$ and $\|\bO_i\|$ is about 1, so $\|\bG^i\|$ is as small as ${\mu k || \bV^{\top} \bYb ||}/(\sqrt{p}n) = \mu k \sint(\bV,\bY)/(\sqrt{p}n) $. High probability can then be established by the trick of re-sampling.

However, in our setting, we have to deal with two major technical obstacles due to deterministic weights.
\begin{itemize}
\item[1.] There is no expectation for $\bD_i$. Since $\|\bD_i\|^2_{\infty}$ can be as large as $\frac{n^2}{\poly(\log n)}$,  $\| \bP_i \|$ can potentially be as large as $\sint(\bY, \bV) \frac{n}{\poly(\log n)}$, which is almost a factor $n$ larger than the bound for random $\bD_i$.  This is clearly insufficient to show the progress.
\item[2.] A priori the norm of $\bO_i = (\bY^{\top} \bD_i \bY)^{-1}$ may be large.  Especially, in the algorithm $\bY$ is given by the alternating minimization steps and giving an upper bound on $\|\bO_i\|$ at all steps seems hard. 
\end{itemize} 

\paragraph{The first issue} 
For this, we exploit the incoherence of $\bY$ and the spectral property of the weight matrix.  
If $\bD_i$ is the identity matrix, then $\bP_i = 0$ which, intuitively, means that there are cancellations between negative part and positive parts. When $\bW$ is expander-like, it will put roughly equal weights on the negative part and the positive part. If furthermore we have that $\bY$ is incoherent (i.e., the negative and positive parts are spread out), then $\bW$ can mix the terms and lead to a  cancellation similar to that when $\bD_i = \bI$. 
More precisely, consider the $(j,j')$-th element in $\bP_i$.
Define a new vector $x \in \Real^n$ such that 
$$
  x_i = (\btV_j)_i (\bY_{j'})_i, \textnormal{~where~} \btV = \bV^{\top} \bYb \bYb^{\top}.
$$
Then we have the cancellation in the form of $\sum_i x_i = 0$. When $\bD_i = \bI$, we simply get $(\bP_i)_{j,j'} = \sum_i x_i = 0$.
When $\bD_i \neq \bI$, we have $(\bP_i)_{j,j'} =  \sum_{s\in [n]} (\bD_i)_s x^{j, j'}_s$.
Now mix over all $i$, we have
\begin{eqnarray*}
\sum_{i \in [n]} ((\bP_i)_{j,j'})^2  & = & \left(\sum_{s\in [n]} (\bD_i)_s x_s\right)^2 = \|\bW x\|^2 \\
& = & \| (\bW - \bE) x \|^2 \quad\quad\quad  (\textrm{since}~\bE x = 0) \\ 
& \le& \gamma^2 n^2 \|x\|^2
\end{eqnarray*}
where in the last step we use the expander-like property of $\bW$ (Assumption $\bold{(A2)}$) to gain the cancellation. 
Furthermore, if $\|\bY_{j'}\|_\infty$ is small, by definition $\|x\|^2$ is also small, so we can get an upper bound on $\sum_{i \in [n]} \|\bP_i\|_F^2$.

Then the problem reduces to maintaining the incoherence of  $\bY$. 
This is taken care of by our clipping step (Algorithm~\ref{alg:clip}), which sets to 0 the rows of $\bY$ that are too large. Of course, we have to show that this will not increase the distance of the clipped $\bY$ and $\bV$. 
The intuition is that we clip only when $\|\bY^i\| \ge 2\mu k/n$. But $\|\bV^i\| \le \mu k/n$, so after clipping, $\bY^i$ only gets closer to $\bV^i$. 

\paragraph{The second issue} 
This is the more difficult technical obstacle, i.e., $\|\bO_i\| = \| (\bY^{\top} \bD_i \bY)^{-1}\|$ can be large. Our key idea is that although individual $\|\bO_i\|$ can indeed be large, this cannot be the case on average. We show that there can just be a few $i$'s such that $\|\bO_i\|$ is large, and they will not contribute much to $\|\bG\|$, so the update can make progress. 

To be more formal, we wish to bound the number of indices $i$ such that $\sigma_{\min}\left(\bY^{\top} \bD_i \bY \right) \le \frac{\lambdal}{4}$. Consider an arbitrary unit vector $a$. Then, 
$$ 
 a\bY^{\top} \bD_i \bY a  = \sum_j a\bY^{\top} (\bD_i)_j \bY a = \sum_j (\bD_i)_j \langle a, \bY^j \rangle^2.
$$
We know that $\bY$ is close to $\bV$, so we rewrite the above using some algebraic manipulation as 
\begin{eqnarray*}
&& \sum_j (\bD_i)_j \langle a, (\bY^j - \bV^j) + \bV^j \rangle^2 \\
& \ge &   \frac{1}{4} \sum_j (\bD_i)_j \langle a, \bV^j \rangle^2 - \frac{1}{3} \sum_j (\bD_i)_j \langle a, \bY^j - \bV^j  \rangle^2 
\end{eqnarray*}

For $j$'s such that $\bY^j$ is close to  $\bV^j$ (denote these $j$'s as $\mathcal{S}_g$), then the terms can be easily bounded since $\bV^{\top} \bD_i \bV \geq \lambdal I$ by assumption. So we only need to consider $j$'s such that $\bY^j$ is far from  $\bV^j$. Since we have incoherence, we know that 
$\|\bY^j - \bV^j\|$ is still bounded in the order of $\mu k/n$. So $a\bY^{\top} \bD_i \bY a$ can be small only when  $\sum_{j\not\in \set{S}_g} (\bD_i)_j$ is large. 

Let $\set{S}$ denote those bad $i$'s. Let $u_{\set{S}}$ be the indicator vector for $\set{S}$ and $u_g$ be the indicator vector for $[n - \set{S}_g]$.
\begin{eqnarray*}
\sum_{i \in \set{S} }\sum_{j\not\in \set{S}_g} (\bD_i)_j  & = & u_{\set{S}}^{\top} \bW u_g \\
& \le & |\set{S}| (n - |\set{S}_g|) + \gamma n \sqrt{|\set{S}| (n - |\set{S}_g|)}
\end{eqnarray*}
where the last step is due to the spectral property of $\bW$. Therefore, there can be only a few $i$'s with large $\sum_{j\not\in \set{S}_g} (\bD_i)_j$. 

\subsection{Proofs of main results}
We only need to show that we can get an initialization close enough to the ground truth so that we can apply the above analysis for the update. 
For SVD initialization, 
\begin{eqnarray*}
[\bX, \bSigma, \bY] 
& = & \textnormal{rank-}k~\SVD(\bW \hp \bMg + \bW  \hp \bNoise).
\end{eqnarray*}
Since $||\bW  \hp \bNoise||_2 \le \delta$ can be regarded as small, the idea is to show that $\bW \hp \bMg$ is close to $\bMg$ in spectral norm and then apply Wedin's theorem~\cite{wedin1972perturbation}. We show this by the spectral gap property of $\bW$ and the incoherence  property of $\bU, \bV$. 

For random initialization, the proof is only a slight modification of that for SVD initialization, because the update requires rather mild conditions on the initialization such that even the random initialization is sufficient (with slightly worse parameters).

\section{Conclusion} \label{sec:conclusion} 

In this paper we presented the first recovery guarantee of weighted low-rank matrix approximation via alternating minimization. Our work generalized prior work on matrix completion, and revealed technical obstacles in analyzing alternating minimization, i.e., the incoherence and spectral properties of the intermediate iterates need to be preserved.  We addressed the obstacles by a simple clipping step, which resulted in a very simple algorithm that almost matches the practical heuristics.

\section*{Acknowledgements} 
This work was supported in part by NSF grants CCF-1527371, DMS-1317308, Simons Investigator Award, Simons Collaboration Grant, and ONR-N00014-16-1-2329. 

\bibliography{biblio}

\begin{thebibliography}{32}
\providecommand{\natexlab}[1]{#1}
\providecommand{\url}[1]{\texttt{#1}}
\expandafter\ifx\csname urlstyle\endcsname\relax
  \providecommand{\doi}[1]{doi: #1}\else
  \providecommand{\doi}{doi: \begingroup \urlstyle{rm}\Url}\fi

\bibitem[Arora et~al.(2016)Arora, Li, Liang, Ma, and Risteski]{arora2015random}
Sanjeev Arora, Yuanzhi Li, Yingyu Liang, Tengyu Ma, and Andrej Risteski.
\newblock A latent variable model approach to pmi-based word embeddings.
\newblock \emph{To appear in Transactions of the Association for Computational
  Linguistics}, 2016.

\bibitem[Bhojanapalli and Jain(2014)]{bhojanapalli2014universal}
Srinadh Bhojanapalli and Prateek Jain.
\newblock Universal matrix completion.
\newblock In \emph{Proceedings of the 31st International Conference on Machine
  Learning (ICML-14)}, pages 1881--1889, 2014.

\bibitem[Bhojanapalli et~al.(2015{\natexlab{a}})Bhojanapalli, Jain, and
  Sanghavi]{DBLP:journals/corr/Bhojanapalli0S14}
Srinadh Bhojanapalli, Prateek Jain, and Sujay Sanghavi.
\newblock Tighter low-rank approximation via sampling the leveraged element.
\newblock In \emph{Proceedings of the Twenty-Sixth Annual ACM-SIAM Symposium on
  Discrete Algorithms}, pages 902--920. SIAM, 2015{\natexlab{a}}.

\bibitem[Bhojanapalli et~al.(2015{\natexlab{b}})Bhojanapalli, Kyrillidis, and
  Sanghavi]{bhojanapalli2015drop}
Srinadh Bhojanapalli, Anastasios Kyrillidis, and Sujay Sanghavi.
\newblock Dropping convexity for faster semi-definite optimization.
\newblock \emph{arXiv preprint arXiv:1509.03917}, 2015{\natexlab{b}}.

\bibitem[Buck et~al.(2014)Buck, Heafield, and van Ooyen]{Buck-commoncrawl}
Christian Buck, Kenneth Heafield, and Bas van Ooyen.
\newblock N-gram counts and language models from the common crawl.
\newblock In \emph{Proceedings of the Language Resources and Evaluation
  Conference}, Reykjavk, Iceland{i}k, Iceland, May 2014.

\bibitem[Candes and Plan(2010)]{candes2010matrix}
Emmanuel~J Candes and Yaniv Plan.
\newblock Matrix completion with noise.
\newblock \emph{Proceedings of the IEEE}, 98\penalty0 (6):\penalty0 925--936,
  2010.

\bibitem[Cand{\`e}s and Recht(2009)]{candes2009exact}
Emmanuel~J Cand{\`e}s and Benjamin Recht.
\newblock Exact matrix completion via convex optimization.
\newblock \emph{Foundations of Computational mathematics}, 9\penalty0
  (6):\penalty0 717--772, 2009.

\bibitem[Cand{\`e}s and Tao(2010)]{candes2010power}
Emmanuel~J Cand{\`e}s and Terence Tao.
\newblock The power of convex relaxation: Near-optimal matrix completion.
\newblock \emph{Information Theory, IEEE Transactions on}, 56\penalty0
  (5):\penalty0 2053--2080, 2010.

\bibitem[Candes et~al.(2015)Candes, Li, and Soltanolkotabi]{candes2015phase}
Emmanuel~J Candes, Xiaodong Li, and Mahdi Soltanolkotabi.
\newblock Phase retrieval via wirtinger flow: Theory and algorithms.
\newblock \emph{Information Theory, IEEE Transactions on}, 61\penalty0
  (4):\penalty0 1985--2007, 2015.

\bibitem[Eriksson and van~den Hengel(2012)]{eriksson2012efficient}
Anders Eriksson and Anton van~den Hengel.
\newblock Efficient computation of robust weighted low-rank matrix
  approximations using the l\_1 norm.
\newblock \emph{Pattern Analysis and Machine Intelligence, IEEE Transactions
  on}, 34\penalty0 (9):\penalty0 1681--1690, 2012.

\bibitem[Feige and Ofek(2005)]{feige2005spectral}
Uriel Feige and Eran Ofek.
\newblock Spectral techniques applied to sparse random graphs.
\newblock \emph{Random Structures \& Algorithms}, 27\penalty0 (2):\penalty0
  251--275, 2005.

\bibitem[Gadian(1982)]{gadian1982nuclear}
David~G Gadian.
\newblock \emph{Nuclear magnetic resonance and its applications to living
  systems}.
\newblock Clarendon Press; Oxford University Press, 1982.

\bibitem[Gillis and Glineur(2011)]{gillis2011low}
Nicolas Gillis and Fran{\c{c}}ois Glineur.
\newblock Low-rank matrix approximation with weights or missing data is
  np-hard.
\newblock \emph{SIAM Journal on Matrix Analysis and Applications}, 32\penalty0
  (4):\penalty0 1149--1165, 2011.

\bibitem[Gross(2011)]{gross2011recovering}
David Gross.
\newblock Recovering low-rank matrices from few coefficients in any basis.
\newblock \emph{Information Theory, IEEE Transactions on}, 57\penalty0
  (3):\penalty0 1548--1566, 2011.

\bibitem[Hardt(2014)]{hardt2014understanding}
Marcus Hardt.
\newblock Understanding alternating minimization for matrix completion.
\newblock In \emph{Foundations of Computer Science (FOCS), 2014 IEEE 55th
  Annual Symposium on}, pages 651--660. IEEE, 2014.

\bibitem[Heiman et~al.(2014)Heiman, Schechtman, and
  Shraibman]{heiman2014deterministic}
Eyal Heiman, Gideon Schechtman, and Adi Shraibman.
\newblock Deterministic algorithms for matrix completion.
\newblock \emph{Random Structures \& Algorithms}, 45\penalty0 (2):\penalty0
  306--317, 2014.

\bibitem[Jain et~al.(2013)Jain, Netrapalli, and Sanghavi]{jain2013low}
Prateek Jain, Praneeth Netrapalli, and Sujay Sanghavi.
\newblock Low-rank matrix completion using alternating minimization.
\newblock In \emph{Proceedings of the forty-fifth annual ACM symposium on
  Theory of computing}, pages 665--674. ACM, 2013.

\bibitem[Keshavan et~al.(2009)Keshavan, Montanari, and Oh]{keshavan2009matrix}
Raghunandan Keshavan, Andrea Montanari, and Sewoong Oh.
\newblock Matrix completion from noisy entries.
\newblock In \emph{Advances in Neural Information Processing Systems}, pages
  952--960, 2009.

\bibitem[Lee and Shraibman(2013)]{lee2013matrix}
Troy Lee and Adi Shraibman.
\newblock Matrix completion from any given set of observations.
\newblock In \emph{Advances in Neural Information Processing Systems}, pages
  1781--1787, 2013.

\bibitem[Levy and Goldberg(2014)]{levy2014neural}
Omer Levy and Yoav Goldberg.
\newblock Neural word embedding as implicit matrix factorization.
\newblock In \emph{Advances in Neural Information Processing Systems}, pages
  2177--2185, 2014.

\bibitem[Li et~al.(2010)Li, Hu, Zhai, and Chen]{li2010improving}
Yanen Li, Jia Hu, ChengXiang Zhai, and Ye~Chen.
\newblock Improving one-class collaborative filtering by incorporating rich
  user information.
\newblock In \emph{Proceedings of the 19th ACM international conference on
  Information and knowledge management}, pages 959--968. ACM, 2010.

\bibitem[Lu et~al.(1997)Lu, Pei, and Wang]{lu1997weighted}
W-S Lu, S-C Pei, and P-H Wang.
\newblock Weighted low-rank approximation of general complex matrices and its
  application in the design of 2-d digital filters.
\newblock \emph{Circuits and Systems I: Fundamental Theory and Applications,
  IEEE Transactions on}, 44\penalty0 (7):\penalty0 650--655, 1997.

\bibitem[Negahban and Wainwright(2012)]{negahban2012restricted}
Sahand Negahban and Martin~J Wainwright.
\newblock Restricted strong convexity and weighted matrix completion: Optimal
  bounds with noise.
\newblock \emph{The Journal of Machine Learning Research}, 13\penalty0
  (1):\penalty0 1665--1697, 2012.

\bibitem[Peeters(1996)]{peeters1996orthogonal}
Ren{\'e} Peeters.
\newblock Orthogonal representations over finite fields and the chromatic
  number of graphs.
\newblock \emph{Combinatorica}, 16\penalty0 (3):\penalty0 417--431, 1996.

\bibitem[Pennington et~al.(2014)Pennington, Socher, and
  Manning]{pennington2014glove}
Jeffrey Pennington, Richard Socher, and Christopher~D Manning.
\newblock Glove: Global vectors for word representation.
\newblock \emph{Proceedings of the Empiricial Methods in Natural Language
  Processing (EMNLP 2014)}, 12:\penalty0 1532--1543, 2014.

\bibitem[Razenshteyn et~al.(2016)Razenshteyn, Song, and
  Woodruff]{iiya2016weighted}
Ilya Razenshteyn, Zhao Song, and David Woodruff.
\newblock Weighted low rank approximations with provable guarantees.
\newblock In \emph{Proceedings of the 48th Annual Symposium on the Theory of
  Computing}, 2016.

\bibitem[Recht(2011)]{recht2011simpler}
Benjamin Recht.
\newblock A simpler approach to matrix completion.
\newblock \emph{The Journal of Machine Learning Research}, 12:\penalty0
  3413--3430, 2011.

\bibitem[Srebro and Jaakkola(2003)]{srebro2003weighted}
Nathan Srebro and Tommi Jaakkola.
\newblock Weighted low-rank approximations.
\newblock In \emph{Proceedings of the 20th International Conference on Machine
  Learning (ICML-03)}, pages 720--727, 2003.

\bibitem[Sun and Luo(2015)]{sun2015Guaranteed}
Ruoyu Sun and Zhi{-}Quan Luo.
\newblock Guaranteed matrix completion via nonconvex factorization.
\newblock In \emph{{IEEE} 56th Annual Symposium on Foundations of Computer
  Science}, pages 270--289, 2015.

\bibitem[Wedin(1972)]{wedin1972perturbation}
Per-{\AA}ke Wedin.
\newblock Perturbation bounds in connection with singular value decomposition.
\newblock \emph{BIT Numerical Mathematics}, 12\penalty0 (1):\penalty0 99--111,
  1972.

\bibitem[Wentzell et~al.(1997)Wentzell, Andrews, and
  Kowalski]{wentzell1997maximum}
Peter~D Wentzell, Darren~T Andrews, and Bruce~R Kowalski.
\newblock Maximum likelihood multivariate calibration.
\newblock \emph{Analytical chemistry}, 69\penalty0 (13):\penalty0 2299--2311,
  1997.

\bibitem[Wikimedia(2012)]{enwiki}
Wikimedia.
\newblock {English Wikipedia dump}.
\newblock
  {http://dumps.wikimedia.org/enwiki/latest/enwiki-latest-pages-articles.xml.bz2},
  2012.
\newblock Accessed Mar-2015.

\end{thebibliography}
\bibliographystyle{plainnat}

\newpage
\onecolumn
\appendix

\section{Preliminaries about subspace distance}  \label{app:preliminary}

Before delving into the proofs, we will prove a few simple preliminaries about subspace angles/distances.  

\begin{defn}[Distance, Principle angle]
Denote the principle angle of $\bY$, $\bV \in \Real^{n \times k}$  as $\theta(\bY,\bV)$. Then for orthogonal matrix $\bY$ (i.e., $\bY^\top \bY = \bI$),
\begin{align*}
\tant(\bY, \bV) & = \| \bYb^{\top} \bV (\bY^{\top} \bV)^{-1}\|_2.
\end{align*}
For orthogonal matrices $\bY$, $\bV$,
\begin{align*}
\cost(\bY, \bV)  & = \sigma_{\min}(\bY^{\top} \bV),\\
\sint(\bY, \bV)  & = \|(\bI  - \bY \bY^{\top}) \bV \|_2 = \|\bYb\bYb^{\top} \bV \|_2 =\|\bYb^{\top} \bV \|_2, \\
\cds(\bY, \bV)  & = \min_{\bQ \in \ortho_{k \times k}} \| \bY \bQ - \bV \|_2
\end{align*}
where $\ortho_{k \times k}$ is the set of $k \times k$ orthogonal matrices.
\end{defn}

\begin{lem}[Equivalence of distance]\label{lem:equiv_distance}

Let $\bY$, $\bV \in \Real^{n \times k}$ be two orthogonal matrices, then we have: 
$$\sint(\bY, \bV) \le \cds(\bY, \bV) \le \sint(\bY, \bV) + \frac{1  - \cost(\bY, \bV)}{\cost(\bY, \bV)} \leq 2 \tant(\bY, \bV).$$

\end{lem}
\begin{proof}[Proof of Lemma \ref{lem:equiv_distance}]
Suppose $$\bQ^* = \argmin_{\bQ \in \ortho_{k \times k}} \| \bY \bQ - \bV \|_2.$$ 
Let's write $\bV = \bY\bQ^* + \bR$, then $ \cds(\bY, \bV)  = \|\bR\|_2$. We have
$$\sint(\bY, \bV) = \|(\bI  - \bY \bY^{\top}) \bV \|_2 = \| \bYb \bYb^{\top} \bR \|_2 \le \| \bR\|_2$$

On the other hand, suppose $\bA \bD \bB^{\top} = \SVD(\bY^{\top} \bV)$, we know that $\sigma_{\min}(\bD) = \sigma_{\min}(\bY^{\top} \bV) = \cost(\bY, \bV) $. Therefore, by $\bA  = \bY^{\top} \bV \bB \bD^{-1} $, $\bA \bB^{\top} \in \ortho_{k \times k}$ we have:
\begin{eqnarray*}
\cds(\bY, \bV) &\le& \|\bY \bA \bB^{\top} - \bV\|_2 = \| \bY \bY^{\top} \bV \bB \bD^{-1} \bB^{\top} - \bV\|_2
\\
&\le& \| \bY \bY^{\top} \bV \bB \bD^{-1} \bB^{\top} - \bY \bY^{\top} \bV \|_2 + \|\bY \bY^{\top} \bV  - \bV\|_2
\\
&\le&\|\bB \bD^{-1} \bB^{\top}  - \bI\|_2 +  \sint(\bY, \bV)  = \| \bD^{-1} - \bI\|_2 +  \sint(\bY, \bV)
\\
&=& \sint(\bY, \bV) + \frac{1  - \cost(\bY, \bV)}{\cost(\bY, \bV)}.
\end{eqnarray*}

Finally, $\sint(\bY, \bV) \le \tant(\bY, \bV)$ and $\frac{1  - \cost(\bY, \bV)}{\cost(\bY, \bV)} \leq \tant(\bY, \bV)$ can be verified by definition, so the last inequality follows.
\end{proof}

For convenience in our proofs we will also use the following generalization of incoherence: 
\begin{defn}[Generalized incoherence] For a matrix $\bA \in \mathbb{R}^{n \times k}$, the \emph{generalized incoherence} $\incoherentpartial(\bA)$ is defined as:
$$\incoherentpartial(\bA) = \max_{i \in [n]} \left\{ \frac{n}{k} \|\bA^i\|_2^2 \right\}$$
\end{defn}

We call it generalized incoherence for obvious reasons: when $\bA$ is an orthogonal matrix, then $\incoherentpartial(\bA) = \mu(\bA)$.

\section{Proofs for alternating minimization with clipping} \label{sec:proof_clip}

We will show in this section the results for our algorithm based on alternating minimization with a clipping step. The organization is as follows. In Section~\ref{s:initclip} we will present the necessary lemmas for the initialization, in Section \ref{s:updateclip} we show the decrease of the potential function after one update step, and in Section \ref{s:thmclip} we will put everything together, and prove our main theorem.

Before starting with the proofs, we will make a remark which will simplify the exposition. 

Without loss of generality, we may assume that 
\begin{equation} 
\delta = \|\bW \hp \bNoise\|_2 \leq \frac{\lambdal\sigma_{\min}(\bMg) }{200  k } \label{eq:deltawlog} 
\end{equation}
Otherwise, we can output the 0 matrix, and the guarantee of all our theorems would be satisfied vacuously. 


\subsection{SVD-based initialization} 

\label{s:initclip}

We want to show that after initialization, the matrices $\bX, \bY$ are close to the ground truth matrix $\bU, \bV$. Observe that $[\bX, \bSigma, \bY] = \SVD(\bW \hp \bM) = \SVD(\bW \hp (\bMg + \bNoise)) = \SVD(\bW \hp \bMg + \bW  \hp \bNoise)$. By our assumptions we know that $||\bW  \hp \bNoise||_2 \le \delta$ which we are thinking of as small, so the idea is to show that $\bW \hp \bMg$ is close to $\bMg$ in spectral norm, then by Wedin's theorem~\cite{wedin1972perturbation} we will have $\bX, \bY$ are close to $\bU, \bV$. We show that $\bW \hp \bMg$ is close to $\bMg$ by the spectral gap property of $\bW$ and the incoherence  property of $\bU, \bV$.

\begin{lem}[Spectral lemma] \label{lem:weighted_unweighted}
Let $\bW$ be an (entry wise non-negative) matrix in $\Real^{n \times n}$ with a spectral gap, i.e. $\bW =  \bE +  \gamma n \bJ \bSigma_{\bW} \bK^{\top} $, where $\bJ, \bK$ are $n \times n$ (column) orthogonal matrices, with $||\bSigma_{\bW}||_2 = 1, \gamma < 1$. Furthermore, for every matrix $\bH \in \Real^{n \times n} $ such that $\bH = \bA \bSigma \bB^{\top}$ ($\bA, \bB$ not necessarily orthogonal, $\bSigma \in \Real^{k \times k}$ is diagonal) we have
$$\leftll \left(\bW - \bE\right) \hp \bH \rightll_2 \le \gamma k \sigma_{\max}(\bSigma) \sqrt{\incoherentpartial(\bA)\incoherentpartial(\bB)}$$
where $\bE$ is the all one matrix.
\end{lem}

\begin{proof}[Proof of Lemma \ref{lem:weighted_unweighted}]

We know that for any unit vectors $x, y \in \Real^{n}$, 
\begin{eqnarray*}
x^{\top} \left( \left(\bW - \bE \right) \hp \bH \right) y &=& \sum_{r = 1}^k \sigma_r x^T \left( \left(\bW - \bE \right) \hp \bA_r \bB_r^{\top}\right) y
\\
&=& \gamma n \sum_{r = 1}^k \sigma_r (\bA_r \hp x)^{\top}  \bJ \bSigma_{\bW}\bK^{\top}  (\bB_r \hp y)
\\
& \le&  \gamma n   \sum_{r = 1}^k \sigma_r || \bA_r \hp x||_2 || \bJ \bSigma_{\bW}\bK^{\top}||_2 ||\bB_r \hp y||_2
\\
& \le& \gamma n    \sum_{r = 1}^k \sigma_r || \bA_r \hp x||_2  ||\bB_r \hp y||_2
\\
&\le&  \gamma n  \sigma_{\max}(\bSigma)\sqrt{ \sum_{r = 1}^k  || \bA_r \hp x||_2^2 } \sqrt{ \sum_{r = 1}^k ||\bB_r \hp y||_2^2}
\\
&\le&  \gamma n  \sigma_{\max}(\bSigma)\sqrt{ \sum_{i = 1}^n x_i^2 ||\bA^i||_2^2 } \sqrt{ \sum_{i = 1}^n y_i^2 ||\bB^i||_2^2}
\\
&\le&  \gamma n \sigma_{\max}(\bSigma)  \sqrt{  \frac{k}{n} \incoherentpartial(\bA) \left( \sum_{i = 1}^n x_i^2\right)}  \sqrt{  \frac{k}{n} \incoherentpartial(\bB)\left( \sum_{i = 1}^n y_i^2\right)} 
\\
&\le&  \gamma \sigma_{\max}(\bSigma) k \sqrt{\incoherentpartial(\bA)\incoherentpartial(\bB)}.
\end{eqnarray*}
The lemma follows from the definition of the operator norm.
\end{proof}

The spectral lemma can be used to prove the initialization condition, when combined with Wedin's theorem.

\begin{lem}[Wedin's Theorem \cite{wedin1972perturbation}]
Let $\bMg, \btM$ be two matrices whose singular values are $\sigma_1, ..., \sigma_n$ and $\sigmat_1, ..., \sigmat_n$, let $\bU, \bV$ and $\bX, \bY$ be the first k singular vectors (left and right) of $\bMg, \btM$ respectively. If $\exists \alpha > 0$ such that $\max_{r = k+1}^n \sigmat_r \le \min_{i = 1}^k \sigma_i - \alpha$, then
$$\max \left\{\sint(\bU, \bX), \sint(\bV, \bY) \right\}\le  \frac{||\btM -\bMg||_2 }{\alpha}.$$

\end{lem}

\begin{lem} \label{lem:initialization1}
Suppose $\bMg, \bW$ satisfy all the assumptions, then for $(\bX, \bSigma, \bY)= \text{rank-}k~\SVD(\bW \hp \bM)$,
we have
$$\max\{ \tant(\bX, \bU), \tant(\bY, \bV) \} \le \frac{4(\gamma \mu k  + \delta)}{\sigma_{\min}(\bMg)}$$

\end{lem}

\begin{proof}[Proof of Lemma \ref{lem:initialization1}]
We know that $$\|\bW \hp \bM - \bMg\|_2 \le ||\bW \hp \bMg - \bMg||_2  + ||\bW \hp \bNoise||_2 \le \gamma \mu k\sigma_{\max}(\bMg) + \delta.$$
Therefore,  by Weyl's theorem, 
$$\max \{\sigma_r(\bW \hp \bM) : k+1 \le r \le n\}\le \gamma \mu k  + \delta \le \frac{1}{2} \sigma_{\min} (\bMg ).$$
where the last inequality holds because of \ref{eq:deltawlog} and the assumption on $\gamma$ in the theorem statement.  

Now, by Wedin's theorem with $\alpha =  \frac{1}{2} \sigma_{\min} (\bMg )$, for $(\bX, \bSigma,\bY)= \text{rank-}k~\SVD(\bW \hp \bM)$, 
$$\max \left\{\sint(\bU, \bX), \sint(\bV, \bY) \right\} \le \frac{2(\gamma \mu k  + \delta)}{\sigma_{\min}(\bMg)} $$ 
Since $\gamma$ and $\delta$ are small enough, so $\text{sin}\theta \leq 1/2$. In this case, we have $\text{tan}\theta \leq 2\text{sin}\theta$, then the lemma follows.
\end{proof}

Finally, this gives us the following guarantee on the initialization: 

\begin{lem}[SVD initialization] \label{lem:initialization}
Suppose $\bMg, \bW$ satisfy all the assumptions.
 $$\cds(\bV, \bY_1) \le  8  k \Delta_1, ~~\rho(\bY_1) \le \frac{2 \mu }{1 - k\Delta_1}$$
where $\Delta_1 = \frac{8 (\gamma \mu k + \delta)}{\sigma_{\min}(\bMg)}$.
\end{lem}

\begin{proof}[Proof of Lemma \ref{lem:initialization}]
First, consider $\btY_1$. 
By Lemma \ref{lem:initialization1} and \ref{lem:equiv_distance}, we get that 
$$\cds(\btY_1, \bV) \le \Delta_1  $$ 
which means  that $\exists \bQ \in \ortho_{k \times k}$, s.t.  
$$ \| \btY_1 \bQ - \bV\|_2 \le \Delta_1 $$ 
hence 
$$ \| \btY_1 \bQ - \bV\|_F \le k \Delta_1 \le \frac{1}{4}$$
where the last inequality follows since $\gamma$ and $\delta$ are small enough.

Next, consider $\bbY_1$. In the clipping step, if $\|\btY^i_1\| \ge \wt =  \frac{2\mu k}{n}$, then $\|\btY^i_1 - \bV^i\|  \ge \frac{\mu k}{n}$, and $\|\bbY^i_1 - \bV^i\|  = \| \bV^i\|  = \frac{\mu k}{n}$. Otherwise, $\bbY^i = \btY^i$. So $$ \| \bbY_1 \bQ - \bV\|_F \le \| \btY_1 \bQ - \bV\|_F \le \frac{1}{4}.$$

Finally, we can argue that $\bY_1$ is close to $\bV$. Let's assume that $\bY_1 = \bbY_1 \bR^{-1}$, for an upper-triangular $\bR$.     
$$
 \sint(\bV, \bY_1) = \| \bVb^\top \bY_1  \|_2 =  \| \bVb^\top (\bbY_1 - \bV \bQ^{-1}) \bR^{-1} \|_2  \leq \|  \bbY_1 \bQ - \bV \|_2 \| \bR^{-1} \|_2 \leq \frac{1}{\sigma_{\min}(\bbY_1)} \|  \bbY_1 \bQ - \bV \|_F 
$$
where the second inequality follows because the singular values of $\bR$ and $\bbY_1$ are the same.   
Note that 
$$
\sigma_{\min}(\bbY_1 ) \ge  \sigma_{\min}(\bV) - \| \bbY_1 - \bV\|_F \ge \sigma_{\min}(\bV) - k \Delta_1 = 1 - k \Delta_1 \geq \frac{1}{2}
$$
So 
$$
 \sint(\bV, \bY_1) \le 2 \|  \bbY_1 \bQ - \bV \|_F \le \frac{1}{2}.
$$  

In this case, we have $\tant(\bV, \bY_1) \le 2\sint(\bV, \bY_1) $ and thus
 $$\cds(\bV, \bY_1) \le 2 \tant(\bV, \bY_1) \le 4 \sint(\bV, \bY_1) \le 8 \|  \bbY_1 \bQ - \bV \|_2 \le 8 \|  \bbY_1 \bQ - \bV \|_F \le 8 k \Delta_1.$$

For $\rho(\bY_1)$, observe that $\bY_1^i = \bbY^i \bR^{-1}$, so
$$ \| \bY_1^i\| \le \| \bbY_1^i \| \| \bR^{-1}\|_2 \le  \frac{\wt}{\sigma_{\min}(\bbY_1)} \le \frac{\wt}{1-k\Delta_1}$$
which leads to the bound. 
\end{proof}

\subsection{Random initialization} 

With respect to the random initialization, the lemma we will need is the following one: 

\begin{lem}[Random initialization]\label{lem:sigma_init}
Let $\bY$ be a random matrix in $\mathbb{R}^{n \times k}$ generated as $\bY_{i,j} = b_{i,j} \frac{1}{\sqrt{n}}$, where $b_{i,j}$ are independent, uniform $\{-1,1\}$ variables.  
Furthermore, let $\|\bW\|_{\infty} \leq \frac{\lambdal n}{k^2 \mu \log^2 n}$. Then, 
with probability at least $1-\frac{1}{n^2}$ over the draw of $\bY$,  
$$\forall i, \sigma_{\min}\left(\bY^{\top} \bD_i \bY\right) \ge \frac{1}{4} \frac{\lambdal}{k \mu}.$$ 
\end{lem}

\begin{proof}[Proof of Lemma \ref{lem:sigma_init}]

Notice that $\bY^{\top} \bD_i \bY  = \sum_j (\bY^j)^{\top} (\bD_i)_j \bY^j $, and each of the terms $(\bY^j)^{\top} (\bD_i)_j \bY^j$ is independent. 
Furthermore, it's easy to see that $\mathbb{E} [(\bY^j)^{\top} (\bD_i)_j (\bY^j)] = \frac{1}{n} (\bD_i)_j$, $\forall j$. 
By linearity of expectation it follows that 
$\mathbb{E}[ \sum_j (\bY^j)^{\top} (\bD_i)_j \bY^j] = \frac{1}{n} \sum_j (\bD_i)_j $.

Now, we claim  $\sum_j (\bD_i)_j \ge \frac{\lambdau n}{k\mu}$. Indeed, by Assumption \textbf{(A3)} we have for any vector $a \in \mathbb{R}^n$  
$$a^{\top} \bV^{\top} \bD_i \bV a = \sum_j (\bD_i)_j \langle \bV^j, a \rangle^2 \ge \lambdal.$$
On the other hand, however, by incoherence of $\bV$, $\sum_j (\bD_i)_j \langle \bV^j, a \rangle^2  \le \sum_j (\bD_i)_j \frac{\mu k}{n} $.  
Hence, $\sum_j (\bD_i)_j \ge \lambdau \frac{n}{k \mu}$. Putting things together, we get
$$\mathbb{E}[ \sum_j (\bY^j)^{\top} (\bD_i)_j \bY^j] \ge \frac{\lambdal}{k\mu}$$

Denote 
$$B := \|(\bY^j)^{\top} (\bD_i)_j \bY^j\|_2 \le \frac{k}{n} (\bD_i)_j \le \frac{\lambdal}{k \mu\log^2 n}$$ 
where the first inequality follows from our sampling procedure, and the last inequality by the assumption that $\|\bW\|_{\infty} \leq \frac{\lambdal n}{k^2 \mu \log^2 n}$. 

Since all the random variables $(\bY^j)^{\top} (\bD_i)_j \bY^j$ are independent, applying Matrix Chernoff we get that
$$ \Pr\left[ \sum_j (\bY^j)^{\top} (\bD_i)_j (\bY^j) \le (1-\delta) \frac{\lambdal}{k \mu}\right] \le n \left(\frac{e^{-\delta}}{(1-\delta)^{(1-\delta)}}\right)^{\frac{\lambdal}{k \mu B}} \le n \left(\frac{e^{-\delta}}{(1-\delta)^{(1-\delta)}}\right)^{\log^2 n} $$ 

Picking $\delta = \frac{3}{4}$, and union bounding over all $i$, with probability at least $1 - \frac{1}{n^2}$, for all $i$,  
$$ \sigma_{\min}\left(\bY^{\top} \bD_i \bY\right)  \ge \frac{1}{4} \frac{\lambdal}{k \mu}$$
as needed. 
\end{proof}

\subsection{Update}
\label{s:updateclip}

We now prove the two key technical lemmas (Lemma~\ref{lem:sigma_min} and Lemma~\ref{lem:projection}) and then use them to prove that the updates make progress towards the ground truth. 
We prove them for $\bY_t$ and use them to show $\bX_t$ improves, while completely analogous arguments also hold when switching the role of the two iterates.
Note that we measure the distance between $\bY_t$ and $\bV$ by $\cds(\bY_t, \bV) = \min_{\bQ \in \ortho_{k\times k}} \| \bY_t \bQ - \bV\|$ where $\ortho_{k\times k}$ is the set of $k \times k$ orthogonal matrices. For simplicity of notations, in these two lemmas, we let $\bY_o = \bY_t \bQ^*$ where $\bQ^* = \argmin_{\bQ \in \ortho_{k\times k}} \| \bY_t \bQ - \bV\|$.

We first show that there can only be a few $i$'s such that the spectral property of $\bY_o^{\top} \bD_i \bY_o$ can be bad, when $\bY_o$ is close to $\bV$. Let $(\bD_i)_j $ be the $j$-th diagonal entry in $\bD_i$, that is, $(\bD_i)_j = \bW_{i,j}$.
\begin{lem}\label{lem:sigma_min}
Let $\bY_o$ be a (column) orthogonal matrix in $\Real^{n \times k}$, and $\epsilon \in (0,1)$. If $ \| \bY_o - \bV \|^2_F \le \frac{\epsilon^3 \lambdal^2 n}{128 \mu k \dmax}$ for $\dmax = \max_{i \in [n]} \sum_j (\bD_{i})_{ j}$, then 
$$\left| \left\{ i  \in [n]\left| \sigma_{\min} (\bY_o^{\top} \bD_i \bY_o) \le (1-\epsilon)\lambdal \right. \right\} \right| \le \frac{1024 \mu^2 k^2 \gamma^2 \dmax}{\epsilon^4\lambdal^3} \| \bV - \bY_o \|_F^2.$$
\end{lem}

\begin{proof}[Proof of Lemma \ref{lem:sigma_min}]
For a value $g > 0$ which we will specify shortly, we call $j \in [n]$ ``good'' if $\| \bY_o^j   - \bV^j \|^2 \le g^2$. Denote the set of ``good'' $j$'s as $\set{S}_g$. 

Then for every unit vector $a \in \Real^k$, 
\begin{eqnarray*}
a^{\top } \bY_o^{\top} \bD_i \bY_o a &=& \sum_{j \in [n]} (\bD_i)_j \langle a, \bY_o^j \rangle^2
\\
&\ge& \sum_{j \in \set{S}_g} (\bD_i)_j \langle a, \bY_o^j \rangle^2
\\
&=&\sum_{j \in \set{S}_g} (\bD_i)_j \left(\langle a, \bV^j \rangle +  \langle a, \bY_o^j - \bV^j \rangle \right)^2
\\
&\ge& (1-\frac{\epsilon}{4})\sum_{j \in \set{S}_g} (\bD_i)_j \langle a, \bV^j \rangle^2 - \frac{4-\epsilon}{\epsilon} \sum_{j \in S_g} (\bD_i)_j  \langle a, \bY_o^j - \bV^j \rangle^2 \quad \\
&& \quad \quad \quad  (\mbox{Using the fact } \forall x, y \in \Real: (x + y)^2 \ge (1-\epsilon_0)x^2 - \frac{1-\epsilon_0}{\epsilon_0}y^2)
\\
&\ge&  (1-\frac{\epsilon}{4})\sum_{j \in \set{S}_g} (\bD_i)_j \langle a, \bV^j \rangle^2 -  \frac{4-\epsilon}{\epsilon} g^2 \sum_{j \in [n]} (\bD_i)_j 
\\
&\ge& (1-\frac{\epsilon}{4})\sum_{j \in [n]} (\bD_i)_j \langle a, \bV^j \rangle^2 -   \frac{\mu k}{n}\sum_{j \in [n] - \set{S}_g} (\bD_i)_j   -  \frac{4-\epsilon}{\epsilon} g^2 \sum_{j \in [n]} (\bD_i)_j 
\end{eqnarray*}

By Assumption \textbf{(A3)}, we know that 
$$\sum_{j \in [n]} (\bD_i)_j \langle a, \bV^j \rangle^2 = a^T \bV^{\top} \bD_i \bV a \ge \sigma_{\min} (\bV^{\top} \bD_i \bV) \ge \lambdal $$

Moreover, recall $\dmax = \max_{i \in [n]} \sum_j (\bD_{i})_{ j}$, so when $g^2 \le \frac{\epsilon^2 \lambdal}{16 \dmax}$, 
$$\frac{4-\epsilon}{\epsilon} g^2 \sum_{j \in [n]} (\bD_i)_j  \le \frac{\epsilon\lambdal}{4}$$

Let us consider now $\sum_{j \in [n] - S_g} (\bD_i)_j  $. 
Define:
$$\set{S} = \left\{ i \in [n] \left| \frac{\mu k}{n}\sum_{j \in [n] - \set{S}_g} (\bD_i)_j  \ge \frac{\epsilon\lambdal}{4} \right.\right\}$$
Then it is sufficient to bound $|\set{S}|$.

For $\set{S}_g$, observe that 

$$\sum_j \| \bV^j - \bY_o^j \|_2^2 = \| \bV- \bY_o\|^2_F $$

Which implies that $$\left| [n] - \set{S}_g  \right| = \text{size}\left([n] - \set{S}_g \right) \le \frac{ \| \bV - \bY_o \|_F^2}{g^2}$$

Let $u_S$ be the indicator vector of $\set{S}$, and $u_g$ be the indicator vector of $[n] - \set{S}_g$, we know that
\begin{eqnarray*}
u_S^{\top} \bW u_g &=& \sum_{i \in \set{S}}\sum_{j \in [n]- \set{S}_g} (\bD_i)_j 
\\
&\ge& \frac{\epsilon \lambdal n}{4 \mu k} |\set{S}|
\end{eqnarray*}

On the other hand, 
\begin{eqnarray*}
u_S^{\top} \bW u_g &= & u_S^{\top} \bE u_g + u_S^{\top} (\bW - \bE) u_g
\\
&\le& |\set{S}| |[n] - \set{S}_g| + \gamma n \sqrt{ |\set{S}| |[n] - \set{S}_g|} 
\end{eqnarray*}

Putting these two inequalities together, we have 
$$|[n] -  \set{S}_g | +  \gamma n  \sqrt{\frac{|[n]  - \set{S}_g|}{|\set{S}|}} \ge \frac{\epsilon \lambdal n}{4 \mu k} $$

Which implies when $|[n] -  \set{S}_g |  \le \frac{\epsilon\lambdal n}{8\mu k} $, we have:
$$|\set{S}| \le \frac{64 \mu^2 k^2 \gamma^2 |[n] -  \set{S}_g | }{\epsilon^2 \lambdal^2} \le \frac{64 \mu^2 k^2 \gamma^2 \| \bV - \bY_o \|_F^2 }{\epsilon^2\lambdal^2 g^2}$$

Then, setting $g^2 = \frac{\epsilon^2 \lambdal}{16 \dmax}$, we have: 
$$\left| \left\{ i  \in [n]\left| \sigma_{\min} (\bY_o^{\top} \bD_i \bY_o) \le (1-\epsilon)\lambdal \right. \right\} \right| \le |\set{S}| \le \frac{1024 \mu^2 k^2 \gamma^2 \dmax}{\epsilon^4\lambdal^3} \| \bV - \bY_o \|_F^2$$
which is what we need. 

\end{proof}

\begin{lem}\label{lem:projection}
Let $\bY_o$ be a (column) orthogonal matrix in $\Real^{n \times k}$.
Then we have
$$\sum_{i \in [n]}\| \bV^{\top}\bYb\bYb^{\top} \bD_i \bY_o \|_2^2  \le \gamma^2 \rho(\bY_o) n k^3 \| \bY_o - \bV \|_2^2 $$

\end{lem}

\begin{proof}[Proof of Lemma \ref{lem:projection}]

We want to bound the spectral norm of $\bV^{\top}\bYb\bYb^{\top} \bD_i \bY_o$, for a fixed $j \in [k]$, let $\bY_j$ be the $j$-th column of $\bY_o$ and $\btV_j$ be the $j$-th column of $\bYb\bYb^{\top} \bV$. 

For fixed $j, j' \in [k]$, consider a new vector $x^{j, j'} \in \Real^n$ such that $x^{j, j'}_i = (\btV_j)_i (\bY_{j'})_i$. 

Note that $\langle \btV_j, \bY_{j'} \rangle = 0$, which implies that $\sum_ix^{j, j'}_i = 0$.

Let us consider $\bV_j^{\top}\bYb\bYb^{\top} \bD_i \bY_{j'}$, we know that 
\begin{eqnarray*}
\bV_j^{\top}\bYb\bYb^{\top} \bD_i \bY_{j'} &=&  \sum_{s \in [n]} (\bD_i)_s (\btV_j)_s (\bY_{j'})_s
\\
&=& \sum_{s\in [n]} (\bD_i)_s x^{j, j'}_s
\end{eqnarray*}

Which implies that 
\begin{eqnarray*}
\sum_{i \in [n]} \left(\sum_{s\in [n]} (\bD_i)_s x^{j, j'}_s\right)^2 &=& \|\bW x^{j, j'}\|_2^2
\\
&=& \| (\bW - \bE) x^{j, j'} \|_2^2 \quad\quad\quad  (\textrm{since}~\bE x^{j, j'} = 0)
\\
&\le& \gamma^2 n^2 \|x^{j, j'}\|_2^2
\end{eqnarray*}

Observe that 
\begin{eqnarray*}
\| x^{j, j'}\|_2^2 &=& \sum_{i \in [n]} (x^{j, j'}_i)^2
\\
&=&  \sum_{i \in [n]} (\btV_j)_i^2 (\bY_{j'})_i^2
\\
&\le& \frac{\rho(\bY_o) k}{n }   \sum_{i \in [n]} (\btV_j)_i^2
\\
&=&  \frac{\rho(\bY_o) k}{n }  \| \btV_j \|_2^2
\\
&\le& \frac{\rho(\bY_o) k}{n }  \| \bYb\bYb^{\top} \bV \|_2^2 
\\
& = & \frac{\rho(\bY_o) k}{n }  \| \bYb\bYb^{\top} (\bY_o - \bV) \|_2^2 
\\
&\le&   \frac{\rho(\bY_o) k}{n } \| \bY_o - \bV \|_2^2.
\end{eqnarray*}

Which implies $$\sum_{i \in [n]} \left(\sum_{s\in [n]} (\bD_i)_s x^{j, j'}_s\right)^2  \le \gamma^2 \rho(\bY_o) n k \| \bY_o - \bV \|_2^2$$

Now we are ready to bound $\bV^{\top}\bYb\bYb^{\top} \bD_i \bY_o$. Note that
\begin{eqnarray*}
\| \bV^{\top}\bYb\bYb^{\top} \bD_i \bY_o \|_2^2 &\le& \| \bV^{\top}\bYb\bYb^{\top} \bD_i \bY_o \|_F^2
\\
&\le& \sum_{j, j' \in [k]} \left(\bV_j^{\top}\bYb\bYb^{\top} \bD_i \bY_{j'}  \right)^2
\\
&=& \sum_{j, j' \in [k]} \left(\sum_{s\in [n]} (\bD_i)_s x^{j, j'}_s\right)^2.
\end{eqnarray*}

This implies that 
\begin{eqnarray*}\sum_{i \in [n]}\| \bV^{\top}\bYb\bYb^{\top} \bD_i \bY_o \|_2^2  &\le& \sum_{i \in [n]}  \sum_{j, j' \in [k]} \left(\sum_{s\in [n]} (\bD_i)_s x^{j, j'}_s\right)^2 \le  \gamma^2 \rho(\bY_o) n k^3 \| \bY_o - \bV \|_2^2.
\end{eqnarray*}
as needed.

%
%
%
%

\end{proof}

We now use the two technical lemmas to prove the guarantees for the iterate after one update step.

\begin{lem}[Update, main]\label{lem:update_main}
Let $\bY$ be a (column) orthogonal matrix in $\Real^{n \times k}$, and $ \cds^2(\bY, \bV) \le \min\{\frac{1}{2}, \frac{\lambdal^2 n}{384 \mu k^2 \dmax} \}$ for $\dmax = \max_{i \in [n]} \sum_j (\bD_{i})_{ j}$.

Define $\btX \leftarrow \argmin_{\bX \in \Real^{n \times k}} \leftll \bM - \bX \bY^{\top} \rightll_{\bW}$. Let $\bbX$ a $n \times k$ matrix such that for each row:
$$\bbX^i =  \left\{ \begin{array}{ll}
         \btX^i & \mbox{if $\| \btX^i \|_2^2 \le \wt = \frac{2\mu k}{n}$}\\
        0 & \mbox{otherwise}.\end{array} \right.$$
Suppose $\bbX$ has QR decomposition $\bbX = \bX \bR$.
Then \\
(1) $\| \bbX- \bU\bSigma \bV^{\top} \bY \|_F^2 \le \Delta_u^2 := \left(\frac{108 \wt \mu^2 k^3 \gamma^2 \dmax}{\lambdal^2} + \frac{160 \gamma^2  \mu \rho(\bY) k^4  }{\lambdal^2}  \right) \cds(\bY, \bV)^2 + \frac{160k}{\lambdal^2 } \| \bW \hp \bNoise \|_2^2.$
\\
(2) If $\Delta_u \leq \frac{1}{8}\sigma_{\min}(\bMg)$, then 
$$\cds(\bU, \bX) \le \frac{8}{\sigma_{\min}(\bMg) - 2 \Delta_u } \Delta_u \textrm{~~and~~} \rho(\bX) \le \frac{4\mu}{\sigma_{\min}(\bMg) - 2 \Delta_u }.$$
\end{lem}

\begin{proof}[Proof of Lemma \ref{lem:update_main}]

(1) By KKT condition, we know that for orthogonal $\bY$, the optimal $\btX$ satisfies
$$\left(\bW \hp \left[ \bM - \btX \bY^{\top}\right] \right) \bY = 0$$
which implies that the $i$-th row $\btX^i$ of $\btX$ is given by
$$\btX^i = \bM^i \bD_i \bY \left(\bY^{\top} \bD_i \bY \right)^{-1} = (\bMg)^i \bD_i \bY \left(\bY^{\top} \bD_i \bY \right)^{-1}  + \bNoise^i \bD_i \bY \left(\bY^{\top} \bD_i \bY \right)^{-1}.$$

Let us consider the first term, by $\bMg = \bU \bSigma \bV^{\top}$, we know that 
\begin{eqnarray*}
(\bMg)^i \bD_i \bY \left(\bY^{\top} \bD_i \bY \right)^{-1}   &=& \bU^i \bSigma \bV^{\top} \bD_i \bY \left(\bY^{\top} \bD_i \bY \right)^{-1}   
\\
&=& \bU^i \bSigma \bV^{\top} (\bY \bY^{\top} + \bYb\bYb^{\top}) \bD_i \bY \left(\bY^{\top} \bD_i \bY \right)^{-1}   
\\
&=&  \bU^i \bSigma \bV^{\top} \bY + \bU^i \bSigma \bV^{\top}\bYb\bYb^{\top} \bD_i \bY \left(\bY^{\top} \bD_i \bY \right)^{-1}   
\end{eqnarray*}
which implies that
$$\btX^i - \bU^i \bSigma \bV^{\top} \bY =  \bU^i \bSigma \bV^{\top}\bYb\bYb^{\top} \bD_i \bY \left(\bY^{\top} \bD_i \bY \right)^{-1}    + \bNoise^i \bD_i \bY \left(\bY^{\top} \bD_i \bY \right)^{-1}  $$

Let us consider set $$\set{S}_1 = \left\{ i  \in [n]\left| \sigma_{\min} (\bY^{\top} \bD_i \bY) \le \frac{\lambdal}{4} \right. \right\} $$ 
Now we have:
\begin{eqnarray*}
\sum_{i \notin \set{S}_1} \left\| \btX^{i} - \bU^i \bSigma \bV^{\top} \bY \right \|_2^2
&\le&  \frac{16}{\lambdal^2}\sum_{i \notin \set{S}_1}  \left( 2\|  \bU^i \bSigma \bV^{\top}\bYb\bYb^{\top} \bD_i \bY  \|_2^2 + 2 \| \bNoise^i \bD_i \bY \|_2^2\right)
\\
&\le&     \frac{32 \mu k \|\bSigma \|_2^2}{n\lambdal^2}\sum_{i \notin \set{S}_1} \| \bV^{\top}\bYb\bYb^{\top} \bD_i \bY  \|_2^2 + \frac{32}{\lambdal^2} \sum_{i \in [n]}  \| \bNoise^i \bD_i \bY \|_2^2
\\
&\le&    \frac{32 \mu k \|\bSigma \|_2^2}{n\lambdal^2}\sum_{i  \in [n]} \| \bV^{\top}\bYb\bYb^{\top} \bD_i \bY  \|_2^2 + \frac{32}{\lambdal^2} \| (\bW \hp \bNoise) \bY \|_F^2
\\
&\le& \Delta_g := \frac{32 \gamma^2  \mu \rho(\bY) k^4    }{\lambdal^2} \cds(\bY,\bV)^2 + \frac{32k}{\lambdal^2} \| (\bW \hp \bNoise) \|_2^2.
\end{eqnarray*}
where the last inequality is due to Lemma \ref{lem:projection}. Note that since $\wt = \frac{2\mu k}{n} \ge  2 \|\bU^i \bSigma \bV^{\top} \bY \|_2^2$, this implies
\begin{eqnarray*}\left| \left\{ i  \in [n] - \set{S}_1\left| \|\btX^i\|_2^2 \ge \wt \right. \right\} \right| &\le& \left| \left\{ i  \in [n] - \set{S}_1\left| \|\btX^i - \bU^i \bSigma \bV^{\top} \bY \|_2^2 \ge \frac{\wt}{2} \right. \right\} \right| 
\le  \frac{2\Delta_g}{\wt}. 
\end{eqnarray*}

Let $\set{S}_2 = \left\{ i  \in [n] - \set{S}_1\left| \|\btX^i\|_2^2 \ge \wt \right. \right\}$, we have:
\begin{eqnarray*}
 \left\| \bbX - \bU \bSigma \bV^{\top} \bY \right \|_F^2 & = & \sum_{i =1}^n \left\| \bbX^{i} - \bU^i \bSigma \bV^{\top} \bY \right \|_2^2 \quad\quad\quad  (\textrm{because }~\|\bbX^{i}\|^2_2 \leq \wt\quad\textrm{and }~\|\bU^i \bSigma \bV^{\top} \bY \|_2^2 \leq	\wt )
\\
& \le & \sum_{i  \in \set{S}_1 \cup \set{S}_2 } 2\wt + \sum_{i \not\in \set{S}_1 \cup \set{S}_2 } \left\| \btX^{i} - \bU^i \bSigma \bV^{\top} \bY \right \|_2^2
\\
& \le & 2\wt (|\set{S}_1| + |\set{S}_2|) + \sum_{i \not\in \set{S}_1 \cup \set{S}_2} \left\| \btX^{i} - \bU^i \bSigma \bV^{\top} \bY \right \|_2^2
\\
&\le & 2\wt  |\set{S}_1|  +  4\Delta_g + \Delta_g.
\end{eqnarray*}

By Lemma \ref{lem:sigma_min}, we know that $|\set{S}_1| \le \frac{54 \mu^2 k^3 \gamma^2 \dmax}{\lambdal^2} \| \bV - \bY \|_2^2 $. Further plugging in $\Delta_g$, we have
\begin{eqnarray*}
 && \left\| \bbX - \bU \bSigma \bV^{\top} \bY \right \|_F^2  \\
&\le & 2\wt \frac{54 \mu^2 k^3 \gamma^2 \dmax}{\lambdal^2} \| \bV - \bY \|_2^2  
 + \frac{160 \gamma^2  \mu \rho(\bY) k^4  }{\lambdal^2} \| \bY - \bV \|_2^2  + \frac{160k}{\lambdal^2} \| (\bW \hp \bNoise) \|_2^2
\\
&=& \left(\frac{108 \wt \mu^2 k^3 \gamma^2 \dmax}{\lambdal^2} + \frac{160 \gamma^2  \mu \rho(\bY) k^4  }{\lambdal^2}  \right) \| \bY - \bV \|_2^2  + \frac{160k}{\lambdal^2 } \| (\bW \hp \bNoise) \|_2^2.
\end{eqnarray*}

(2) Denote $\bB = \bSigma \bV^{\top} \bY$. Then, 
\begin{align*}
 \sint(\bU, \bX) = \| \bUb^\top \bX  \|_2 =  \| \bUb^\top (\bbX - \bU \bB) \bR^{-1} \|_2  \leq \|  \bbX - \bU \bB \|_2\| \bR^{-1} \|_2 = \frac{1}{\sigma_{\min}(\bbX)} \|  \bbX - \bU \bB \|_2
\end{align*}
Since $ \| \bbX - \bU \bB \|_2  \le \Delta_u$, we have 
$$
\sigma_{\min}(\bbX) \ge  \sigma_{\min}(\bU \bB) - \Delta_u = \sigma_{\min}(\bSigma \bV^{\top} \bY) - \Delta_u \geq \sigma_{\min}(\bMg) \cost(\bY, \bV)  - \Delta_u.
$$
By the assumption $\cost(\bY, \bV) \geq 1/2$, so
\begin{align*}
 \sint(\bU, \bX) \le \frac{2}{\sigma_{\min}(\bMg) - 2 \Delta_u } \Delta_u.
\end{align*}
When $\Delta_u \leq \frac{1}{8} \sigma_{\min}(\bMg)$, the right hand side is smaller than $1/3$, so $\cost(\bU, \bX) \ge 1/2$, and thus $\tant(\bU, \bX) \le 2  \sint(\bU, \bX)$.
Then the statement on $\cds(\bU, \bX)$ follows from $\cds(\bU, \bX)  \le 2\tant(\bU, \bX) \le 4\sint(\bU, \bX)$. 

Finally, observe that $\bX^i = \bbX^i \bR^{-1}$, so
$$ \| \bX^i\|_2 \le \| \bbX^i \|_2 \| \bR^{-1}\|_2 \le  \frac{\wt}{\sigma_{\min}(\bbX)}$$
which leads to the bound. 
\end{proof}

\subsection{Putting everything together: proofs of the main theorems}
\label{s:thmclip} 

Finally, in this section we put things together and prove the main theorems. 

We first proceed to the SVD-initialization based algorithm: 

\begin{oneshot}{Theorem~\ref{thm:alt_min}}
If $\bMg, \bW$ satisfy assumptions \textbf{(A1)-(A3)}, and
\begin{align*}
\gamma  = O\left( \min\left\{ \sqrt{\frac{n}{D_1}}\frac{\lambdal  }{\ckappa \mu^{3/2} k^2} , 
\frac{\lambdal }{\ckappa^{3/2}\mu k^2}    \right\} \right),
\end{align*}
then after $O(\log(1/\epsilon))$ rounds Algorithm~\ref{alg:main} with initialization from Algorithm~\ref{alg:initial} outputs a matrix $\btM$ that satisfies
\begin{align*}  
  ||\btM - \bMg||_2 \le O\left(\frac{k \ckappa}{\lambdal} \right) ||\bW \hp \bNoise||_2 + \epsilon.
\end{align*}
The running time is polynomial in $n$ and $\log(1/\epsilon)$.
\end{oneshot}

\begin{proof}[Proof of Theorem~\ref{thm:alt_min}]
We first show by induction $\cds(\bX_t, \bU)  \le \frac{1}{2^t} +  70 \frac{k}{\lambdal\sigma_{\min}(\bMg)}\delta$ for $t > 1$,  and $\cds(\bY_t, \bU) \le \frac{1}{2^t} +  70\frac{k}{\lambdal\sigma_{\min}(\bMg)}\delta$ for $t \ge 1$. 

First, by Lemma~\ref{lem:initialization}, $\bY_1$ satisfies
$$\cds(\bV, \bY_1) \le 8 k \Delta_1 = \frac{64 k (\gamma \mu k  + \delta)}{\sigma_{\min}(\bMg)}.$$
Since $\gamma = O\left( \frac{1}{\ckappa k^2 \mu }\right)$, the base case follows. 
Now proceed to the inductive step and prove the statement for $t+1$ assuming it is true for $t$. Now we can apply Lemma~\ref{lem:update_main}. By taking the constants within the $O(\cdot)$ notation for $\gamma$ sufficiently small and by the inductive hypothesis, we have 
$$
	\left(\frac{108 \wt \mu^2 k^3 \gamma^2 \dmax}{\lambdal^2} + \frac{160 \gamma^2  \mu \rho(\bY_1) k^4  }{\lambdal^2}  \right) \leq \frac{1}{100}\sigma^2_{\min}(\bMg)
$$
and 
$$
\Delta_u \le \frac{1}{8} \sigma_{\min}(\bMg). 
$$
By Lemma~\ref{lem:update_main}, we get 
\begin{flalign*}
 &\cds(\bU, \bX_{t+1}) \le  \frac{2}{ \sigma_{\min}(\bMg) - 2 \Delta_u} \Delta_u \le \frac{8}{3 \sigma_{\min}(\bMg)} \Delta_u &\\
 &=\frac{8}{3 \sigma_{\min}(\bMg)} \sqrt{\left(\frac{108 \wt \mu^2 k^3 \gamma^2 \dmax}{\lambdal^2} + \frac{160 \gamma^2  \mu \rho(\bY_1) k^4  }{\lambdal^2}  \right) \cds^2(\bU, \bX_{t})  + \frac{160k}{\lambdal^2} \delta^2}  &\\ 
&\leq \frac{8}{3 \sigma_{\min}(\bMg)} \left(\sqrt{\left(\frac{108 \wt \mu^2 k^3 \gamma^2 \dmax}{\lambdal^2} + \frac{160 \gamma^2  \mu \rho(\bY_1) k^4  }{\lambdal^2}  \right) \cds^2(\bY_t, \bV) } + \sqrt{\frac{160k}{\lambdal^2} \delta^2}\right) \quad\quad\mbox{(using $\sqrt{a+b} \leq \sqrt{a} + \sqrt{b}$)} &\\ 
&\leq \frac{1}{2}\cds(\bY_t, \bV) + \frac{35 \sqrt{k} }{\lambdal \sigma_{\min}(\bMg)} \delta &\\
\end{flalign*}
so the statement also holds for $t+1$. This completes the proof for bounding $\cds(\bX_t, \bU) $ and $\cds(\bY_t, \bV)$. 

Given the bounds on $\cds(\bX_t, \bU) $ and $\cds(\bY_t, \bV)$, we are now ready to prove the theorem statement.
For simplicity, let $\bbX$ denote $\bbX_{T+1}$ and $\bY$ denote $\bY_T$, so the algorithm outputs $\btM = \bbX \bY$. 

By Lemma~\ref{lem:update_main}, 
\begin{align*}
\| \bbX- \bU\bSigma \bV^{\top} \bY \|_F^2 & \le \Delta_u^2 := \left(\frac{108 \wt \mu^2 k^3 \gamma^2 \dmax}{\lambdal^2} + \frac{160 \gamma^2  \mu \rho(\bY) k^4  }{\lambdal^2}  \right) \cds(\bY, \bV)^2 + \frac{160k}{\lambdal^2 } \| \bW \hp \bNoise \|_2^2.
\end{align*}

Plugging the choice of $\gamma$ and noting $\wt = \frac{2\mu k}{n}$ and $\rho(\bY)  = O(\mu/\sigma_{\min}(\bMg))$, 
we have

\begin{align*}
\| \bbX- \bU\bSigma \bV^{\top} \bY \|_F^2 & \le \Delta_u^2 = O\left( \cds(\bY, \bV)^2 \right) + \frac{160k}{\lambdal^2 } \| \bW \hp \bNoise \|_2^2
\end{align*}
which leads to
\begin{align*}
\| \bbX- \bU\bSigma \bV^{\top} \bY \|_F & \le \Delta_u \le O\left( \cds(\bY, \bV) \right)+ \frac{16\sqrt{k}}{\lambdal } \| \bW \hp \bNoise \|_2.
\end{align*}

Now consider $\|\bMg - \btM\|_2 = \|\bMg - \bbX \bY^\top\|_2 $.  By definition, we know that there exists $\bQ$ such that  $\bY  = \bV\bQ + \Delta_y$ where $\| \Delta_y\|_2 = O(\cds(\bY, \bV)) $. Also, let $\bR = \bbX - \bU\bSigma \bV^{\top} \bY $.
\begin{align*}
\btM - \bMg  & = \left[  \bU \bSigma \bV^\top(\bV \bQ + \Delta_y) + \bR \right] ( \bV \bQ + \Delta_y)^\top - \bU \bSigma \bV^\top \\
& = \bU \bSigma \bQ \Delta_y^\top + \bU\bSigma\bV^\top \Delta_y ( \bV \bQ + \Delta_y)^\top + \bR ( \bV \bQ + \Delta_y)^\top \\
& = \bU \bSigma \bQ \Delta_y^\top + \bU\bSigma\bV^\top \Delta_y \bY^\top + \bR \bY^\top.
\end{align*}

Therefore,
\begin{align*}
\|\btM - \bMg\|_2  & \le  \|\bU \bSigma\|_2 \|\bQ\|_2 \|\Delta_y\|_2 + \|\bU\bSigma\bV^\top\|_2 \|\Delta_y\|_2\| \bY\|_2 + \|\bR\|_2 \|\bY\|_2 \\
& \le 2 \|\Delta_y\|_2 + \|\bR\|_2 \\
& \le O\left( \cds(\bY, \bV) \right)+ \frac{16\sqrt{k}}{\lambdal } \| \bW \hp \bNoise \|_2.
\end{align*}
Combining this with the bound on $\cds(\bY_T, \bV)$, the theorem then follows.
\end{proof}

Next, we show the main theorem for random initialization: 

\begin{oneshot}{Theorem~\ref{thm:alt_min_rand}}[Main, random initialization]   
Suppose $\bMg, \bW$ satisfy assumptions \textbf{(A1)-(A3)} with
\begin{align*}
\gamma  &= O\left( \min\left\{ \sqrt{\frac{n}{D_1}}\frac{\lambdal }{ \ckappa\mu^{2} k^{5/2}} , 
\frac{\lambdal }{ \ckappa^{3/2} \mu^{3/2} k^{5/2}}    \right\} \right), 
\\ 
\|\bW\|_{\infty} & = O\left(\frac{\lambdal n}{k^2 \mu \log^2 n}\right), 
\end{align*}
where $D_1 = \max_{i \in [n]} \|\bW^i\|_1$. 
Then after $O(\log(1/\epsilon))$ rounds Algorithm~\ref{alg:main} using initialization from Algorithm~\ref{alg:rinitial} outputs a matrix $\btM$ that with probability at least $1-1/n^2$ satisfies
\begin{align*}  
  \|\btM - \bMg\|_2 \le O\left(\frac{k\ckappa}{\lambdal}\right) \|\bW \hp \bNoise\|_2 + \epsilon.
\end{align*}
The running time is polynomial in $n$ and $\log(1/\epsilon)$.
\end{oneshot}

\begin{proof}[Proof of Theorem~\ref{thm:alt_min_rand}]

Let $\bY$ be initialized using the random initialization algorithm \ref{alg:rinitial}. Consider applying the proof in Lemma \ref{lem:update_main}, with $\set{S}_1$ being modified to be  
$$\set{S}_1 = \left\{ i  \in [n]\left| \sigma_{\min} (\bY^{\top} \bD_i \bY) \le \frac{\lambdal}{4 \mu k} \right. \right\} $$ 
But with this modification, $\set{S}_1 = \emptyset$, with high probability. Then the same calculation from Lemma \ref{lem:update_main} (which now doesn't need to use Lemma \ref{lem:sigma_min} at all since $\set{S}_1 = \emptyset$) gives 
$$ \left\| \bbX - \bU \bSigma \bV^{\top} \bY \right \|_F^2 \leq \Delta_g \mu k $$
But following part (2) of the same Lemma, we get that if $\Delta_g \mu k < \frac{1}{8} \sigma_{\min}(M^{*})$, 
$$ \cds(\bU, \bX) \le \frac{2}{\sigma_{\min}(\bMg) - 2 \Delta_g \mu k } \Delta_g \mu k $$

So, in order to argue by induction in \ref{thm:alt_min} exactly as before, we only need to check that after the update step for $\bX$, $\cds(\bU, \bX)$ is small enough to apply Lemma \ref{lem:update_main} for later steps. Indeed, we have:
$$\cds(\bU, \bX) \le \frac{2}{\sigma_{\min}(\bMg) - 2 \Delta_g \mu k } \Delta_g \mu k \le \sqrt{\min\left\{\frac{1}{2}, \frac{\lambdal^2 n}{384 \mu k^2 \dmax} \right\}} $$ 

Noticing that $\Delta_g$ has a quadratic dependency on $\gamma$, we see that if 
\begin{align*}
\gamma  = O\left( \min\left\{ \sqrt{\frac{n}{D_1}}\frac{\lambdal \sigma_{\min}(\bMg) }{ \mu^{2} k^{5/2}} , 
\frac{\lambdal \sigma^{3/2}_{\min}(\bMg)}{\mu^{3/2} k^{5/2}}    \right\} \right),
\end{align*}
the inequality is indeed satisfied. 

With that, the theorem statement follows. 

\end{proof} 
\subsection{Estimating $\sigma_{\max}(\bMg)$} 

Finally, we show that we can estimate $\sigma_{\max}(\bMg)$ up to a very good accuracy, so that we can apply our main theorems to matrices with arbitrary $\sigma_{\max}(\bMg)$. This is quite easy: the estimate of it is just $\|\bW \hp \bM\|_2$. Then, the following lemma holds:

\begin{lem} It $\gamma = o(\frac{1}{k \mu})$ and $\delta = \|\bW \hp \bNoise\|_2 = o(\sigma_{\max}(\bMg) )$ then $ \|\bW \hp \bM\|_2  = (1 \pm o(1))(\sigma_{\max}(\bMg)) $ 
\label{l:estimateinitial}
\end{lem} 
\begin{proof}
We proceed separately for the upper and lower bound. 

For the upper bound, we have 
\begin{align*}
\|\bW \hp \bM\|_2 &= \|\bW \hp \bMg + \bW \hp \bNoise\|_2 \leq \|\bW \hp \bMg\|_2 + \|\bW \hp \bNoise\|_2 \\
&\le \|(\bW - \bE) \hp \bMg\|_2 + \|\bE \hp \bMg\|_2 + \|\bW \hp \bNoise\|_2 \\
&\le \gamma k \mu \sigma_{\max}(\bMg) + \sigma_{\max}(\bMg) + \delta \le (1+o(1)) \sigma_{\max}(\bMg). \quad\quad\quad\mbox{(by Lemma \ref{lem:weighted_unweighted})}
\end{align*} 
 
For the lower bound, completely analogously we have 
\begin{align*}
\|\bW \hp \bM\|_2 &= \|\bW \hp \bMg + \bW \hp \bNoise\|_2 \geq \|\bW \hp \bMg\|_2 - \|\bW \hp \bNoise\|_2 \\
&\ge \|\bE \hp \bMg\|_2 - \|(\bW - \bE) \hp \bMg\|_2  - \|\bW \hp \bNoise\|_2 \\
&\ge \sigma_{\max}(\bMg) - \gamma k \mu \sigma_{\max}(\bMg) - \delta \geq (1-o(1)) \sigma_{\max}(\bMg) \quad\quad\quad\mbox{(by Lemma \ref{lem:weighted_unweighted})}
\end{align*} 

which finishes the proof. 
\end{proof} 

Given this, the reduction to the case $\sigma_{\max}(\bMg) \leq 1$ is obvious: first, we scale the matrix $\bM$ down by our estimate of $\sigma_{\max}(\bMg)$ and run our algorithm with, say, four times as many rounds. After this, we rescale the resulting matrix $\btM$ by our estimate of $\sigma_{\max}(\bMg)$, after which the claim of Theorems \ref{thm:alt_min} and \ref{thm:alt_min_rand} follows. 

\section{An alternative approach: alternating minimization with SDP whitening} \label{app:proof_sdp}

Our main results build on the insight that the spectral property only need to hold in an average sense. However, we can even make sure that the spectral property holds at each step in a strict sense by a whitening step using SDP and Rademacher rounding. This is presented a previous version of the paper, and we keep this result here since potentially it can be applied in some other applications where similar spectral properties are needed and is thus of independent interest. 

\begin{algorithm}[!t]
\caption{Main Algorithm ($\AlgMain$)}\label{alg:main_whitening}
\begin{algorithmic}[1]
\REQUIRE Noisy observation $\bM$, weight matrix $\bW$, rank $\rank$, number of iterations $T$
\STATE $(\bX_1,\bY_1) = \AlgInitial(\bM, \bW)$, $d_1=\frac{1}{8 k \sqrt{\log n} } + \frac{64\sqrt{k} \delta}{ \lambdal \sigma_{\min}(\bMg)} $  
\STATE $\bbY_1 \leftarrow \AlgWhitening(\bY_1, \bW, d_1, \lambdal, \lambdau, \mu, k)$
\FOR{$t = 1, 2, ..., T$}
\STATE $d_{t + 1}  = \frac{1}{2^{t+1}} \frac{1}{8 k \sqrt{\log n} } + \frac{64\sqrt{k}}{ \lambdal \sigma_{\min}(\bMg)} \delta$
\STATE $\bX_{t + 1} \leftarrow \argmin_{\bX \in \Real^{n \times k}} \leftll \bM - \bX \bbY^{\top}_t \rightll_{\bW}$
\STATE $\btX_{t+ 1} \leftarrow \QR(\bX_{t + 1} )$
\STATE $\bbX_{ t + 1} \leftarrow \AlgWhitening(\btX_{t + 1}, \bW, d_{t+1}, \lambdal, \lambdau, \mu, k)$
\STATE $\bY_{t + 1} \leftarrow \argmin_{\bY \in \Real^{n \times k}} \leftll \bM - \bbX_{t + 1} \bY^{\top} \rightll_{\bW}$
\STATE $\btY_{t + 1} \leftarrow \QR(\bY_{t + 1} )$
\STATE $\bbY_{ t + 1} \leftarrow \AlgWhitening(\btY_{t + 1}, \bW, d_{t+1}, \lambdal, \lambdau, \mu, k)$
\ENDFOR
\STATE $\bbSigma \leftarrow \argmin_{\bSigma} \|\bW \hp (\bM - \bbX_{ T + 1} \bSigma \bbY_{T + 1}^{\top} ) \|_2$
\ENSURE $\btM = \bbX_{ T + 1} \bbSigma \bbY_{T + 1}^{\top}$
\end{algorithmic}
\end{algorithm}

\begin{algorithm}[!t]
\caption{Whitening ($\AlgWhitening$)}\label{alg:whitening}
\begin{algorithmic}[1]
\REQUIRE  orthogonal matrix $\btX \in \Real^{n \times k}$, weight $\bW$, distance $d$, spectral barriers $\lambdal,\lambdau$, incoherency $\mu$, rank $k$. 
\STATE Solve the following convex programing on the matrices $\bR \in \Real^{n \times k}$ and $\{ \bA_r  \in \Real^{k\times k} \}_{r = 1}^n$:
\begin{align*}
& ||\bR - \btX ||_2 \le d \\
& ||\btX^{\top}(\bR - \btX) + (\bR - \btX)^{\top}\btX ||_2 \le d^2 \\
& ||\btXb^{\top} \bR||_2 \le d  \\
& (\bR^{r})^{\top}\bR^{r} \preceq \bA_r, \forall r \in [n]\\
& \trace(\bA_r) \le \frac{\mu k}{n}, \forall r \in [n]\\
& \sum_{r = 1}^n \bA_r = \bI  \\
& \lambdal \bI \preceq \sum_{r = 1}^n  \bW_{i, r}\bA_r  \preceq \lambdau \bI, \forall i \in [n]
\end{align*}
\STATE $\forall r \in [n], \bX^r \sim \Rad(\bR^r, \bA_r - (\bR^{r})^{\top}\bR^{r})$.
\STATE $\bbX = \QR(\bX)$, $\bX \in \Real^{n \times k} $ whose rows are $\bX^r$. 
\ENSURE  $\bbX$. (may need $O(\log(1/\alpha))$ runs to succeed with probability $1-\alpha$; see text)
\end{algorithmic}
\end{algorithm}

The whitening step (see Algorithm~\ref{alg:whitening}) is a convex (actually semidefinite) relaxation followed by a randomized rounding procedure. We explain each of the constraints in the semidefinite program in turn. The first three constraints control the spectral distance between $\bbX$ and $\btX$. The next two constraints control the incoherency, and the rest are for the spectral ratio.  The solution of the relaxation is then used to specify the mean and variance of a Rademacher (random) vector, from which the final output of the whitening step is drawn. Here a Rademacher vector is defined as:

\begin{defn}[Rademacher random vector]
A random vector $x \in \Real^{k}$ is a \emph{Rademacher random vector} with mean $\mu$ and variance $\Sigma \succeq 0$ (denoted as $x \sim \Rad(\mu, \Sigma)$), if $x = \mu + \bS \sigma$ where $\bS$ is a $k \times k$ symmetric matrix such that $\bS^2 = \Sigma$, $\sigma \in \Real^{k}$ is a vector where each entry is i.i.d Rademacher random variable. 
\end{defn}

We use this type of random vector to ensure that if $x \sim \Rad(\mu, \Sigma)$, then $\E[x] = \mu, \E[xx^{\top}] = \mu \mu^{\top} + \Sigma$.
Since the desired properties of the output of whitening can be tested (see Lemma~\ref{lem:whitening}), we can repeat the whitening step $O(\log(1/\alpha))$ times to get high probability $1- \alpha$. In the rest of the paper, we will just assume that it is repeated sufficiently many times (polynomial in $n$ and $\log(1/\epsilon)$) so that Algorithm~\ref{alg:main_whitening} succeeds with probability $1-1/n$.

We now present the analysis for this algorithm. The SVD initialization has been analyzed, so we focus on the update step and the whitening step.

\paragraph{Note} Since our algorithm will output matrix $\btM$ such that $||\btM - \bMg||_2 \approx O\left(\frac{k^{3/2}\sqrt{\log n}}{\lambdal \sigma_{\min}(\bMg)} \right) ||\bW \hp \bNoise||_2 $  and $||\bMg||_2 = 1, \lambdal \le 1$, therefore, without lose of generality we can assume that $||\bW \hp \bNoise||_2 \le \frac{\lambdal \sigma_{\min}(\bMg)}{k^{3/2}\sqrt{\log n} } $, otherwise we can just output zero matrix.

\subsection{Update}

We want to show that after every round of $\AlgMain$, we move our current matrices $\bX, \bY$ closer to the optimum. We will show that $\btX \leftarrow \argmin_{\bX \in \Real^{n \times k}} \leftll \bM - \bX \bY^{\top} \rightll_{\bW}$ is a noisy power method update: $\btX = \bMg \bY + \bG$ where $||\bG||_2$ is small. 

For intuition, note that if $||\bG||_2 = 0$, that is, $\btX = \bMg \bY$, then we know that $\tant(\btX, \bU) = 0$, so within one step of update we will be already hit into the correct subspace. We will show when $||\bG||_2$ is small we still have that $\tant(\btX, \bU) $ is progressively decreasing. Then, in order to show $||\bG||_2$ is small, we need to make sure we start from a good $\bY$ as assumed in Lemma \ref{lem:update}.

First, we show that when $\bG$ is small, then $\tant(\btX, \bU)$ is small.

\begin{lem}[Distance from OPT]\label{lem:power_method}

Let $\bMg = \bU \bSigma \bV^T \in \Real^{n \times n}$ be the singular value decomposition of a rank-$k$ matrix $\bMg$, let $\bY \in \Real^{n \times k}$ be an orthogonal matrix, $\btX = \bMg \bY + \bG$, then we have
$$\tant(\btX, \bU)  \le \frac{|| \bG ||_2 }{\cost(\bY, \bV)\sigma_{\min}(\bSigma)  - || \bG||_2   }. $$
\end{lem}

\begin{proof}[Proof of Lemma \ref{lem:power_method}]
By definition,
\begin{eqnarray*}
\tant(\btX, \bU) &=& ||\bUb^{\top} \btX (\bU^{\top} \btX)^{-1} ||_2
\\
&=& || \bUb^{\top} (\bMg \bY + \bG) (\bU^{\top} (\bMg \bY+ \bG))^{-1} ||_2
\\
&\le& || \bUb^{\top} \bG ( \bSigma \bV^{\top} \bY + \bU^{\top} \bG)^{-1}||_2
\\
&\le& || \bUb^{\top} \bG ||_2 ||( \bSigma \bV^{\top} \bY + \bU^{\top} \bG)^{-1}||_2
\\
&\le& || \bUb^{\top} \bG ||_2  || (\bV^{\top} \bY)^{-1}||_2 ||( \bSigma  + \bU^{\top} \bG  (\bV^{\top} \bY)^{-1})^{-1}||_2
\\
&\le& || \bUb^{\top} \bG ||_2 \frac{1}{\cost(\bY, \bV)} \sigma_{\min}^{-1} \left(  \bSigma  + \bU^{\top} \bG  (\bV^{\top} \bY)^{-1} \right).
\end{eqnarray*}
For the last term, we have
$$\sigma_{\min}(\bSigma  + \bU^{\top} \bG  (\bV^{\top} \bY)^{-1}) \ge \sigma_{\min} (\bSigma) - \sigma_{\max}\left( \bU^{\top} \bG  (\bV^{\top} \bY)^{-1} \right)\ge  \sigma_{\min}( \bSigma) - \frac{ ||\bU^{\top} \bG||_2 }{\cost (\bY, \bV)}.$$ 
Therefore,
\begin{eqnarray*}
\tant(\btX, \bU) &\le& \frac{|| \bUb^{\top} \bG ||_2 }{\cost(\bY, \bV)\left( \sigma_{\min}(\bSigma)  - \frac{ ||\bU^{\top} \bG||_2 }{\cost (\bY, \bV)} \right) }
\\
&=&\frac{|| \bUb^{\top} \bG ||_2 }{\cost(\bY, \bV)\sigma_{\min}(\bSigma)  - ||\bU^{\top} \bG||_2   }
\\
&\le& \frac{|| \bG ||_2 }{\cost(\bY, \bV)\sigma_{\min}(\bSigma)  - || \bG||_2}
\end{eqnarray*}
completing the proof.
\end{proof}

Now we show that if $\bY$ has nice properties as stated in Lemma~\ref{lem:update}, then $\bG$ is small. 
Recall the following notation: for a matrix $\bA$, let $\rho(\bA)$ be defined as $\max_i \{\frac{n}{k} ||\bA^i||_2^2 \}$. 

\begin{lem}[Bounding $||\bG||_2$]\label{lem:weighted_optimal}

Let $\bMg = \bU \bSigma \bV^{\top} \in \Real^{n \times n}$ be the singular value decomposition of a rank-$k$ matrix $\bMg$, $\bM = \bMg + \bNoise$ be the noisy observation, and let $\bW, \bMg$ satisfy the conditions of Theorem~\ref{thm:main}. Let $\bY \in \Real^{n \times k}$ be an orthogonal matrix. For
$$\btX =\argmin_{\bX}{ || \bM - \bX \bY^{\top}||_{\bW} }$$
we have $\btX = \bMg \bY + \bG$ where $$||\bG||_2 \le \max_{i \in [n]} \left \{ \gamma k^{3/2} \frac{\sqrt{\incoherentpartial(\bU)\incoherentpartial(\bY) }}{\sigma_{\min} ( \bY^{\top} \bD_i  \bY )} \sint(\bY, \bV) + \frac{ \sqrt{k} || \bW \hp \bNoise ||_2}{\sigma_{\min} ( \bY^{\top} \bD_i  \bY )}  \right\}.$$

\end{lem}

\begin{proof}[Proof of Lemma \ref{lem:weighted_optimal}]

By taking the derivatives of $|| \bM - \bX \bY^{\top}||_{\bW}$ w.r.t.\ $\bX$, we know that the optimal solution $\btX$ satisfies $(\bW \hp [ \bM -  \btX \bY^{\top} ]) \bY = 0$. Plugging in $\btX = \bMg \bY + \bG$, we get
$$ ( \bW \hp [\bM -  \bMg \bY \bY^{\top}] ) \bY = ( \bW \hp [\bG \bY^{\top}]) \bY. $$
Since $\bM = \bMg + \bNoise$ and  $\bI = \bY \bY^{\top} + \bYb  \bYb^{\top}$, the above equation is
$$ ( \bW \hp [\bG \bY^{\top}]) \bY = ( \bW \hp [ \bMg \bYb \bYb^{\top}] ) \bY  + (\bW \hp  \bNoise ) \bY. $$
So for any $i \in [n]$ (recall that $[\cdot]^i$ is the $i$-th row) 
 \begin{align} \label{eqn:weighted_optimal1}
[( \bW \hp [\bG \bY^{\top}]) \bY]^i =  [( \bW \hp [ \bMg \bYb \bYb^{\top}] ) \bY]^i +  [(\bW \hp \bNoise) \bY]^i.
\end{align}

Note that for every matrix $\bS \in \Real^{n \times n}$, for $\bD_i = \diag(\bW^i)$ we have 
$$[\bW \hp \bS]^i = \bS^i \bD_i.$$
Applying this to (\ref{eqn:weighted_optimal1}) leads to 
$$\bG^i \bY^{\top} \bD_i  \bY = (\bMg)^i \bYb \bYb^{\top}  \bD_i \bY   +   [(\bW \hp \bNoise)]^i \bY.$$
Since $ (\bMg)^i \bYb \bYb^{\top} \bI \bY = 0$,
$$\bG^i \bY^{\top} \bD_i  \bY = (\bMg)^i \bYb \bYb^{\top}  (\bD_i - \bI) \bY   +   [(\bW \hp \bNoise)]^i \bY.$$  
This gives us 
\begin{align}
\bG^i  =  (\bMg)^i \bYb \bYb^{\top} (  \bD_i - \bI) \bY( \bY^{\top} \bD_i  \bY )^{-1} +   [(\bW \hp \bNoise)]^i \bY (\bY^{\top} \bD_i  \bY)^{-1}. \label{eqn:weighted_optimal2}
\end{align}

Now we turn to bound the operator norm of $\bG$. By definition, it suffices to bound $|| a^{\top} \bG b ||_2$ for any two  \emph{unit vectors} $a \in \Real^{n \times 1}, b \in \Real^{k \times 1}$ (note that for a scalar $s$, $||s||_2 = |s|$). By (\ref{eqn:weighted_optimal2}), 
\begin{eqnarray*}
|| a^{\top} \bG b ||_2 &=& \leftll \sum_{i = 1}^n a_i (\bMg)^i \bYb \bYb^{\top} (  \bD_i - \bI) \bY( \bY^{\top} \bD_i  \bY )^{-1}b +   \sum_{i = 1}^n a_i  [(\bW \hp \bNoise)]^i  \bY (\bY^{\top} \bD_i  \bY)^{-1}b \rightll_2
\\
&\le&  \underbrace{\leftll \sum_{i = 1}^n a_i  (\bMg)^i \bYb \bYb^{\top} (  \bD_i - \bI) \bY( \bY^{\top} \bD_i  \bY )^{-1}b \rightll_2}_{T1} + \underbrace{\leftll    \sum_{i = 1}^n a_i  [(\bW \hp \bNoise)]^i  \bY (\bY^{\top} \bD_i  \bY)^{-1}b \rightll_2}_{T2}.
\end{eqnarray*}
  
In the following, we bound the two terms $T1$ and $T2$ respectively.

\paragraph{(Bounding $T1$)}  
Let $\bQ =  \bSigma \bV^{\top} \bYb \bYb^{\top}$. We have 
$$(\bMg)^i \bYb \bYb^{\top} = \bU^i \bQ \text{~~and~~} ||\bQ||_2 \le ||\bV^{\top} \bYb||_2 = \sint(\bY, \bV).$$ 
Also let $\bB$ denote the matrix whose $i$-th column is $\bB_i = ( \bY^{\top} \bD_i  \bY )^{-1}b$.  Then $T1$ becomes
\begin{eqnarray*}
T1 &=& \leftll \sum_{i = 1}^n a_i  (\bMg)^i \bYb \bYb^{\top} (  \bD_i - \bI) \bY  ( \bY^{\top} \bD_i  \bY )^{-1}b \rightll_2 
 \\
 &=& \leftll  \sum_{i = 1}^n a_i \bU^i \bQ(  \bD_i - \bI) \bY\bB_i\rightll_2 
 \\
 &=&  \leftll  \sum_{i = 1}^n \sum_{r = 1}^k  (a_i \bU_{i, r}) \bQ^r (  \bD_i - \bI) \bY\bB_i \rightll_2
 \\
&=&  \leftll  \sum_{r = 1}^k  \sum_{i = 1}^n  (a_i \bU_{i, r}) \bQ^r(  \bD_i - \bI) \bY\bB_i \rightll_2
\\
  &=& \leftll  \sum_{r = 1}^k \sum_{i , j = 1}^n  (a_i \bU_{i, r})(  \bW_{i, j} - 1) \bQ_{r, j} \bY^j \bB_i \rightll_2
\end{eqnarray*}
where the last equality is because $\bQ^r(  \bD_i - \bI) \bY = \sum_{j = 1}^n (  \bW_{i, j} - 1) \bQ_{r, j} \bY^j $.

Now denote $\alpha_{i, r} = a_i \bU_{i, r}$ and $\alpha_r = (\alpha_{1, r}, ..., \alpha_{n, r})^{\top}$.
\begin{eqnarray*}
T1 &=& \leftll  \sum_{r = 1}^k  \sum_{i , j = 1}^n  (a_i \bU_{i, r})(  \bW_{i, j} - 1) \bQ_{r, j} \bY^j \bB_i \rightll_2 
\\
&=& \leftll  \sum_{r = 1}^k  \alpha_r^\top [(\bW - \bE) \hp (\bB^\top \bY^\top)] \bQ_r \rightll_2 
\\
&\le & \sum_{r = 1}^k \leftll  \alpha_r^\top [(\bW - \bE) \hp (\bB^\top \bY^\top)] \bQ_r \rightll_2 
\\
&=&  \sum_{r = 1}^k \leftll   \alpha_r  \rightll_2 \leftll  (\bW - \bE) \hp (\bB^\top \bY^\top) \rightll_2 \leftll  \bQ_r \rightll_2.
\end{eqnarray*}
Clearly, for $\leftll  \bQ_r \rightll_2$ we have 
$$\leftll  \bQ_r \rightll_2 \le \leftll  \bQ \rightll_2 \leq \sint(\bY, \bV).$$ 
For $\leftll   \alpha_r  \rightll_2$, we have
\begin{eqnarray*}
\sum_{r = 1}^k ||\alpha_r||_2 &\le&  \sqrt{k} \sqrt{\sum_{r = 1}^k ||\alpha_r||_2^2}
\\
& = & \sqrt{k} \sqrt{\sum_{r = 1}^k\sum_{i = 1}^n a_i^2  \bU_{i, r}^2 }
\\
&=&  \sqrt{k}\sqrt{ \sum_{i = 1}^n\left( a_i^2 \left(\sum_{r = 1}^k  \bU_{i, r}^2 \right) \right)}
\\
&\le & \sqrt{k} \sqrt{\frac{k \incoherentpartial(\bU)}{n} \left(\sum_{i = 1}^n a_i^2 \right)} 
\\
& = & k\sqrt{\frac{ \incoherentpartial(\bU)}{n}}.
\end{eqnarray*}
For $\leftll  (\bW - \bE) \hp (\bB^\top \bY^\top) \rightll_2$, we can apply the spectral lemma (Lemma~\ref{lem:weighted_unweighted}) to get
$$ 
\leftll  (\bW - \bE) \hp (\bB^\top \bY^\top) \rightll_2 \le \gamma k \sqrt{\incoherentpartial(\bB^\top)\incoherentpartial(\bY)}.
$$ 
We have $||\bB_i ||_2 = ||( \bY^{\top} \bD_i  \bY )^{-1}b ||_2 \leq \frac{1}{\sigma_{\min} ( \bY^{\top} \bD_i  \bY ) } $, so 
$$\incoherentpartial(\bB^\top) \leq \max_{i\in[n]} \left\{ \frac{n/k}{\sigma^2_{\min} ( \bY^{\top} \bD_i  \bY ) }  \right\}$$ 
and 
$$ 
\leftll  (\bW - \bE) \hp (\bB^\top \bY^\top) \rightll_2 \le \max_{i\in[n]} \left\{  \frac{\gamma \sqrt{kn\incoherentpartial(\bY)} }{\sigma_{\min} ( \bY^{\top} \bD_i  \bY ) } \right\}.
$$ 

Putting together, we have
\begin{eqnarray}
T1 & \le   & k\sqrt{\frac{ \incoherentpartial(\bU)}{n}} \times \max_{i\in[n]} \left\{  \frac{\gamma \sqrt{kn\incoherentpartial(\bY)} }{\sigma_{\min} ( \bY^{\top} \bD_i  \bY ) } \right\}\times  \sint(\bY, \bV)
\nonumber \\
& \le & \max_{i\in[n]} \left\{  \frac{\gamma k^{3/2}\sqrt{\incoherentpartial(\bY)\incoherentpartial(\bU)} }{\sigma_{\min} ( \bY^{\top} \bD_i  \bY ) } \sint(\bY, \bV) \right\}. \label{eqn:T1}
\end{eqnarray}

\paragraph{(Bounding $T2$)}  
Recall that $\bB$ denote the matrix whose $i$-th column is $\bB_i = ( \bY^{\top} \bD_i  \bY )^{-1}b$.  
\begin{eqnarray*}
T2 & = & \leftll    \sum_{i = 1}^n a_i  [(\bW \hp \bNoise)]^i \bY (\bY^{\top} \bD_i  \bY)^{-1}b \rightll_2
\\
& = & \leftll    \sum_{i = 1}^n a_i  [(\bW \hp \bNoise)]^i \bY \bB_i \rightll_2
\\ 
 & = & \leftll    \sum_{i = 1}^n a_i  [(\bW \hp \bNoise)]^i  \sum_{r=1}^k \bY_r  \bB_{r,i} \rightll_2
\\ 
 & = & \leftll   \sum_{r=1}^k   \sum_{i = 1}^n a_i  [(\bW \hp \bNoise)]^i \bY_r  \bB_{r,i} \rightll_2.
\end{eqnarray*}

Now denote $\beta_{i,r} = a_i \bB_{r,i}$ and $\beta_r = (\beta_{r,1}, \beta_{r,2}, \dots, \beta_{r,n})^\top$. 
\begin{eqnarray*}
T2 & = &    \leftll \sum_{r=1}^k  \sum_{i = 1}^n a_i  [(\bW \hp \bNoise)]^i \bY_r  \bB_{r,i} \rightll_2 
\\
& = & \leftll  \sum_{r=1}^k   \sum_{i = 1}^n \beta_{i,r}  [(\bW \hp \bNoise)]^i \bY_r  \rightll_2 
\\
& = &  \leftll \sum_{r=1}^k   \beta_r^\top  (\bW \hp \bNoise) \bY_r  \rightll_2 
\\
& \le &  \sum_{r=1}^k \leftll   \beta_r^\top  (\bW \hp \bNoise) \bY_r  \rightll_2 
\\
& \leq &  \sum_{r=1}^k   \leftll \beta_r  \rightll_2    \leftll \bW \hp \bNoise \rightll_2  \leftll\bY_r  \rightll_2.
\end{eqnarray*}
We have $\leftll\bY_r  \rightll_2 = 1$.  For $\leftll \beta_r  \rightll_2$, we have
\begin{eqnarray*}
\sum_{r=1}^k   \leftll \beta_r  \rightll_2   &\le&  \sqrt{k} \sqrt{\sum_{r = 1}^k ||\beta_r||_2^2}
\\
&=&  \sqrt{k}\sqrt{ \sum_{r = 1}^k \sum_{i = 1}^n \left( a_i^2 \bB_{r,i}^2 \right)}
\\
&=&  \sqrt{k}\sqrt{ \sum_{i = 1}^n a_i^2 \sum_{r = 1}^k \bB_{r,i}^2 }
\\
&=&  \sqrt{k}\sqrt{ \sum_{i = 1}^n a_i^2 ||\bB_i||_2^2 }.
\end{eqnarray*}
We have $||\bB_i ||_2 = ||( \bY^{\top} \bD_i  \bY )^{-1}b ||_2 \leq \frac{1}{\sigma_{\min} ( \bY^{\top} \bD_i  \bY ) } $ and $\sum_{i = 1}^n a_i^2   = 1$, so 
\begin{eqnarray*}
\sum_{r=1}^k    \leftll \beta_r  \rightll_2  &\le &  \sqrt{k}\sqrt{ \sum_{i = 1}^n a_i^2 ||\bB_i||_2^2 } 
\\
 &\le &  \max_{i\in[n]} \left\{  \frac{\sqrt{k}}{\sigma_{\min} ( \bY^{\top} \bD_i  \bY ) } \right\}.
\end{eqnarray*}

Putting together, we have
\begin{eqnarray}\label{eqn:T2}
T2  &\le & \max_{i\in[n]} \left\{  \frac{\sqrt{k}}{\sigma_{\min} ( \bY^{\top} \bD_i  \bY ) }  \leftll \bW \hp \bNoise \rightll_2 \right\}.
\end{eqnarray}

The lemma follows from combining (\ref{eqn:T1}) and (\ref{eqn:T2}).
\end{proof}

Now we have all the ingredients to prove the update lemma.

\begin{lem}\label{lem:update}
Suppose $\bMg, \bW$ satisfy all the assumptions, column orthogonal matrix $\bY \in \Real^{n \times k}$ is $(5k\mu)$-incoherent, and  for all $i \in [n]$, $\bD_i = \diag(\bW^i)$ satisfies 
$$\frac{1}{4} \lambdal \bI \preceq \bY^{\top} \bD_i \bY \preceq 4\lambdau \bI.$$ 
Then $\btX \leftarrow \argmin_{\bX \in \Real^{n \times k}} \leftll \bM - \bX \bY^{\top} \rightll_{\bW}$ satisfies 
$$\tant(\btX, \bU) \le \frac{ 1}{16 k \sqrt{\log n}} \tant(\bY, \bV) + \frac{16 k \delta}{ \lambdal \sigma_{\min}(\bMg)}.$$
\end{lem}

\begin{proof}[Proof of Lemma \ref{lem:update}]
By Lemma~\ref{lem:power_method} and Lemma~\ref{lem:weighted_optimal}, we have
\begin{align}
\tant(\btX, \bU) & \le \frac{|| \bG ||_2 }{\cost(\bY, \bV)\sigma_{\min}(\bMg)  - || \bG||_2   }, \label{eqn:update1} \\
||\bG||_2 & \le  \max_{i \in [n]} \left \{ \gamma k^{3/2} \frac{\sqrt{\incoherentpartial(\bU)\incoherentpartial(\bY) }}{\sigma_{\min} ( \bY^{\top} \bD_i  \bY )} \sint(\bY, \bV) + \frac{\sqrt{k} || \bW \hp \bNoise ||_2}{\sigma_{\min} ( \bY^{\top} \bD_i  \bY )}  \right\}. \label{eqn:update2}
\end{align}
By the assumptions $\bY^{\top} \bD_i \bY \succeq \frac{\lambdal}{4} \bI$, $\incoherentpartial(\bU) \le \mu, \incoherentpartial(\bY)\le 5 k \mu$, in (\ref{eqn:update2}) we have: 
$$|| \bG||_2\le  \frac{4\sqrt{5}\gamma k^{2} \mu}{\lambdal} \sint(\bY, \bV) +  \frac{4 \sqrt{k} \delta}{\lambdal }.$$
Plugging this in (\ref{eqn:update1}), and noting that 
$$\gamma \le \frac{\lambdal}{128\sqrt{5}k^{3} \mu \sqrt{\log n}}, ||\bG||_2 \le \frac{1}{4 }  \sigma_{\min}(\bMg), \cost(\bY, \bV) \ge \frac{1}{2},$$
we get
$$\tant(\btX, \bU)  \le 2\frac{|| \bG ||_2 }{\cost(\bY, \bV)\sigma_{\min}(\bMg)  } \le \frac{ 1}{16 k\sqrt{\log n}} \tant(\bY, \bV) +  \frac{ 16 \sqrt{k}\delta}{\lambdal \sigma_{\min}(\bMg)}$$
as needed.
\end{proof}

\subsection{Whitening}
What remains is to show that the whitening step can make sure that $\bY$ has good incoherency and $\bO_i$ has the desired spectral property.
Recall that the whitening step consists of a SDP relaxation and a new rounding scheme to fix $\bY$ whenever $\bY^{\top} \bD_i \bY$ having very small singular values. Intuitively, we want to get through the SDP relaxation, an $\bR$ close to $\bV$ and $\bA_{r} \approx (\bV^r)^{\top} \bV^r \in \Real^{k \times k}$, so that we'd have the incoherency of $\bR$ is close to $\mu(\bV)$ which is bounded by $\mu$, and 
$\sum_{r = 1}^{n} \bW_{i, r} \bA_{r} \approx \bV^{\top} \bD_i \bV \succeq \lambdal \bI. $ (Note one can \emph{not} simply say when $\tant(\bY, \bU) \le d$, then $\bY^{\top} \bD_i \bY$ is close to $\bV^{\top} \bD_i \bV$. This is because $|| \bD_i||_2$ can be as large as $\frac{n}{\poly(\log n)}$ in our case, however, $||\bV^{\top} \bD_i \bV||_2 = O(1)$.)
 
The key observation is that our randomized rounding outputs a $n \times k$ random matrix $\bX$ such that $\E[\bX^r] = \bR^r$ ($\bX^r$ is the i-th row of $\bX$), $\E[(\bX^r)^{\top} (\bX^r) ] = \bA_r$, with the variance of $(\bX^r)$ bounded by $\bA_r - (\bR^r)^{\top}  \bR^r$. Therefore, 
$$\E[\bX^r] = \bR^r, \quad \E[\bX^{\top} \bD_i \bX] =  \sum_{r = 1}^{n} \bW_{i, r} \bA_{r} \succeq  \lambdal \bI $$

Thus, $\bX$ is incoherent (Note $||\bX^r||_2^2 = \trace[(\bX^r)^{\top} (\bX^r)] $) and $||(\bX^{\top} \bD_i \bX)^{-1}||_2$ is small in expectation. we can apply matrix concentration bound on $\bX$ to show that the above values actually concentrate on the expectation, thus the output matrix $\bbX = \QR(\bX)$ will have the required properties. 

\begin{lem}[Whitening] \label{lem:whitening} 
Suppose $\btX$ is $\mu$-incoherent and satisfies $\tant(\btX, \bU) \le \frac{d}{2}$ where $d \le \frac{1}{4k\sqrt{\log n}}$. Then $\bbX \leftarrow \AlgWhitening(\btX ,\bW, d, \lambdal, \lambdau, \mu, k)$ satisfies with high probability: \\
(1). $\frac{1}{4} \lambdal \bI \preceq \bbX^{\top} \bD_i \bbX \preceq 4\lambdau \bI$; \\
(2). $\bbX$ is $(5k\mu)$-incoherent; \\
(3). $\tant(\bbX, \bU) \le 4dk\sqrt{\log n} $.
\end{lem}

%
%
%

As a preliminary to showing whitening works, we need to introduce a new type of random variables and a new matrix concentration bound. Another natural distribution to use is a Gaussian random vector  $y \sim \mathcal{N}(\mu, \Sigma)$. The advantage of a Rademacher vector $x$ is that $\|x\|_2$ is always bounded, which facillitates proving concentration bounds.

\begin{lem}[Matrix Concentration]\label{lem:matrix_concentration}
Let $\{x_i\}_{i = 1}^n$ be independent Rademacher random vectors in $\Real^k$ with $x_i \sim \Rad(a_i, \bDelta_i)$, let $ \trace(\bDelta)_{\max}  = \max_{i \in [n]} \{ \trace(\bDelta_i ) \}, (||a||^2_2)_{\max}  = \max_{i \in [n]} \{ ||a_i ||_2^2 \}$, $|| \sum_{i = 1}^n \bDelta_i || \le \Delta$, then for every $t \ge 0$, 
\\
\begin{eqnarray*}
\Pr\left[\leftll \sum_{i = 1}^n x_i x_i^{\top}  - \E\left[ \sum_{i = 1}^n x_i x_i^{\top} \right] \rightll \ge t \right] \le \exp\left\{- \frac{t^2}{c_1  + c_2  t}\right\}
\end{eqnarray*}
where
\begin{eqnarray*}
c_1 & = & [2 \trace(\bDelta)_{\max} + (3 + k)(||a||^2_2)_{\max} ] \Delta, \\
c_2 & = &  ( k + 1)  \trace(\bDelta)_{\max}  +  2\sqrt{ k (||a||_2^2)_{\max} \Delta } .
\end{eqnarray*}

\end{lem}

%
%
%

\begin{proof}[Proof of Lemma \ref{lem:matrix_concentration}]

Let $\bS_i \in \Real^{k \times k}$ be a matrix such that $\bS_i^2 = \bDelta_i$, 
Note that $$\E[x_i x_i^{\top}] = \E[(a_i + \bS_i \sigma)(a_i + \bS_i \sigma)^{\top}] = a_i a_i^{\top} + \bS_i  \E[\sigma \sigma^{\top}] \bS_i^{\top} =  a_i a_i^{\top} + \bDelta_i$$
We first move the random variable to center at zero: consider $y_i = x_i - a_i$,  define $\bY_i =x_i x_i^{\top}  - \E\left[ x_i x_i^{\top} \right]   = x_i x_i^{\top} - (a_i a_i^{\top} + \bDelta_i) = y_i y_i^{\top} + a_i y_i^{\top} + y_i a_i^{\top} -  \bDelta_i$, we have: 
$\E[\bY_i] = 0$. By $y_i$ and $-y_i$ being identically distributed, we obtain
$$\E[||y_i||_2^2 y_i a_i^{\top}] = 0, \E[\langle a_i ,y_i \rangle y_i y_i^{\top}] = 0$$

Therefore, using the fact that $\E[y_i] = 0$, and $\E[y_i y_i^{\top}] = \Delta_i$, we can calculate that
\begin{eqnarray*} 
\E[\bY_i^2] &=& \E[(y_i y_i^{\top} + a_i y_i^{\top} + y_i a_i^{\top} -  \bDelta_i)^2]
\\
&=& \E[||y_i||_2^2 y_i y_i^{\top} + \langle a_i, y_i \rangle y_i y_i^{\top} + ||y_i||_2^2 y_i a_i^{\top} - y_i y_i^{\top} \bDelta_i
\\
&& + ||y_i||_2^2  a_i y_i^{\top} +  \langle a_i, y_i \rangle a_i y_i^{\top} + ||y_i||_2^2 a_i a_i^{\top} - a_i y_i^{\top} \bDelta_i
\\
&& + \langle a_i, y_i \rangle  y_i y_i^{\top} + ||a_i||_2^2 y_i y_i^{\top} + \langle a_i, y_i \rangle y_i a_i^{\top} - y_i a_i^{\top} \bDelta_i
\\
&& - \bDelta_i y_i y_i^{\top} - \bDelta_i a_i y_i^{\top} - \bDelta_i y_i a_i^{\top} + \bDelta_i^2]
\\
&=& \E[||y_i||_2^2 y_i y_i^{\top}] + a_i a_i^{\top} \E[||y_i||_2^2] + ||a_i||_2^2 \E[y_i y_i^{\top}] + \E[\langle a_i, y_i \rangle(a_i y_i^{\top} + y_i a_i^{\top})] - \bDelta_i^2
\\
&=& \E[||y_i||_2^2 y_i y_i^{\top}] + \trace(\bDelta_i) a_i a_i^{\top} + ||a_i||_2^2 \bDelta_i + a_i a_i^{\top} \bDelta_i +  \bDelta_i a_i a_i^{\top} - \bDelta_i^2
\end{eqnarray*}

Furthermore, 
\begin{eqnarray*}
\E[||y_i||_2^2 y_i y_i^{\top}]  &=& \E[(\sigma^{\top} \bDelta_i \sigma) \bS_i\sigma \sigma^{\top} \bS_i^{\top}] 
\\
&=&  \bS_i \E[(\sigma^{\top} \bDelta_i \sigma)   \sigma \sigma^{\top}  ] \bS_i^{\top}
\end{eqnarray*}

On the other hand, For $u \not= v$: 
\begin{eqnarray*}
(\E[(\sigma^{\top} \bDelta_i \sigma)   \sigma \sigma^{\top}  ])_{u, v}  &=& \E\left[ \sum_{p, q} \sigma_p (\bDelta_i)_{p, q} \sigma_q \sigma_u \sigma_v \right]
\\
&=& \sum_{p, q}(\bDelta_i)_{p, q} \E[\sigma_p \sigma_q \sigma_u \sigma_v]
\\
&=& 2(\bDelta_i)_{u, v}
\end{eqnarray*}

For $u = v$: 
\begin{eqnarray*}
(\E[(\sigma^{\top} \bDelta_i \sigma)   \sigma \sigma^{\top}  ])_{u, u}  &=& \E\left[ \sum_{p, q} \sigma_p (\bDelta_i)_{p, q} \sigma_q \sigma_u \sigma_u \right]
\\
&=& \sum_{p, q}(\bDelta_i)_{p, q} \E[\sigma_p \sigma_q \sigma_u^2]
\\
&=&  \sum_{p}(\bDelta_i)_{p, p} = \trace(\bDelta_i)
\end{eqnarray*}

Therefore, $$\E[(\sigma^{\top} \bDelta_i \sigma)   \sigma \sigma^{\top}  ] \preceq \trace(\bDelta_i) \bI + 2 \bDelta_i$$
\begin{eqnarray*}
\E[||y_i||_2^2 y_i y_i^{\top}]  \preceq 2\bDelta_i^2 + \trace(\bDelta_i) \bDelta_i
\end{eqnarray*}

Therefore, by $\bDelta_i^2 \preceq \trace(\bDelta_i) \bDelta_i, a_i a_i^{\top}  \bDelta_i +\bDelta_i a_i a_i^{\top} \preceq 2||a_i||_2^2 \bDelta_i$, we obtain
\begin{eqnarray*}
\sum_{i = 1}^n\E[\bY_i^2]  &\preceq& \sum_{i = 1}^n \left(\bDelta_i^2 + \trace(\bDelta_i) \bDelta_i + \trace(\bDelta_i) a_i a_i^{\top} + ||a_i||_2^2 \bDelta_i + a_i a_i^{\top} \bDelta_i +  \bDelta_i a_i a_i^{\top}   \right)
\\
&\preceq& [2 \trace(\bDelta)_{\max} + 3(||a||^2_2)_{\max} ] \Delta \bI + \sum_{i = 1}^n  \trace(\bDelta_i)  a_i a_i^{\top}
\\
& \preceq&  [2 \trace(\bDelta)_{\max} + 3(||a||^2_2)_{\max} ] \Delta  \bI+ \sum_{i = 1}^n ||a_i||_2^2 \trace(\bDelta_i)  \bI
\\
& \preceq&  [2 \trace(\bDelta)_{\max} + 3(||a||^2_2)_{\max} ] \Delta \bI + (||a||_2^2)_{\max} \trace\left(\sum_{i = 1}^n\bDelta_i \right)  \bI
\\
& \preceq&  [2 \trace(\bDelta)_{\max} + 3(||a||^2_2)_{\max} ] \Delta  \bI+ k(||a||_2^2)_{\max}\Delta  \bI
\end{eqnarray*}

Moreover, 
\begin{eqnarray*}
||\bY_i||_2 &\le& ||\bDelta_i||_2 + ||y_i y_i^{\top} ||_2 + ||a_i y_i^{\top}||_2 +  ||y_i a_i^{\top}||_2
\\
&=&  ||\bDelta_i||_2 + 2 || a_i \sigma^{\top} \bS_i^{\top}||_2 + ||\bS_i \sigma \sigma  \bS_i^{\top}||_2
\\
&\le& ||\bDelta_i||_2  + 2\sqrt{ k (||a||_2^2)_{\max} \Delta } + k ||\bDelta_i||_2
\\
&\le& ( k + 1)  \trace(\bDelta)_{\max}  +  2\sqrt{ k (||a||_2^2)_{\max} \Delta }
\end{eqnarray*}
where the last inequality is due to $||\sigma \sigma^{\top} ||_2 \le k$.

The lemma then follows by the matrix Bernstein inequality.
\end{proof}

Now we are ready to prove the lemma for the whitening step.

\begin{oneshot}{Lemma~\ref{lem:whitening}}
Suppose $\bMg, \bNoise, \bW$ satisfy all assumptions, $\mu$-incoherent column orthogonal matrix $\btX \in \Real^{n \times k}$ is close to $\bU$: $\tant(\btX, \bU) \le \frac{d}{2}$ where $d \le \frac{1}{4k\sqrt{\log n}}$, then $\bbX \leftarrow \AlgWhitening(\btX ,\bW, d, \lambdal, \lambdau, \mu, k)$ satisfies with high probability: 
(1). For all $i \in [n]$,  let $\bD_i = \diag(\bW^i)$, then $\frac{1}{4} \lambdal \bI \preceq \bbX^{\top} \bD_i \bbX \preceq 4\lambdau \bI$; (2). $\bbX$ is $(5k\mu)$-incoherent; (3). $\tant(\bbX, \bU) \le 4dk\sqrt{\log n} $.
\end{oneshot}

\begin{proof}[Proof of Lemma \ref{lem:whitening}]
Firstly we need to show that there is a feasible solution to our SDP relaxation, and then we need to show that the output has the desired properties stated in the lemma.

\paragraph{(Existence of a feasible solution)}
To be specific, we want to show that $\bR = \bU \bQ$, $\bA_r =  \bQ^{\top }(\bU^r)^{\top}\bU^r \bQ$ is a feasible solution to the SDP for some \emph{orthogonal} matrix $\bQ \in \Real^{k \times k}$. 

Clearly, by setting $\bR$ and $\bA_r$ as above, we automatically satisfy: $$(\bR^r)^{\top} \bR^r \preceq \bA_r$$
$$\trace(\bA_r) = ||\bR^r||_2^2 =  || \bQ \bU^r||_2^2  =  ||\bU^r||_2^2  \le \frac{\mu k}{n}$$ 
$$\sum_{r = 1}^n \bA_r = \sum_{r = 1}^n  \bQ^{\top }(\bU^r)^{\top} \bU^r \bQ = \bQ^{\top}\bU^{\top} \bU \bQ = \bI$$
$$\lambdal \bI \preceq \sum_{r = 1}^n \bW_{i, r}\bA_r = \sum_{r = 1}^n  \bW_{i, r}\bQ^{\top}(\bU^r)^{\top} \bU^r \bQ= \bQ^{\top}\left( \bU^{\top}  \diag(\bW^i) \bU \right) \bQ\preceq \lambdau \bI$$

So we only need to show that there exists orthogonal $\bQ $ that $\bU\bQ$ satisfies the distance constraints:
\begin{eqnarray*}
||\btXb^{\top}  \bU \bQ ||_2 & \le  & d, \\
||\bU  \bQ- \btX||_2 & \le & d, \\
||\btX^{\top} (\bU  \bQ - \btX)  + (\bU  \bQ - \btX)^{\top} \btX ||_2 & \le & d^2.
\end{eqnarray*}
Note that $\sint(\btX, \bU) \le \tant(\btX, \bU)$, so 
$$||\btXb^{\top} \bU \bQ||_2 = ||\btXb^{\top} \bU||_2 =  \sint(\btX, \bU) \le \tant(\btX, \bU)  \le d/2 \le d.$$
Moreover,  when $\tant(\btX, \bU)  = \frac{d}{2} \le \frac{1}{2}$,  
$$ \frac{1-\cost(\btX, \bU)}{\cost(\btX, \bU) } \leq \sint(\btX, \bU),$$
and thus by Lemma \ref{lem:equiv_distance}, $\cds(\btX, \bU) \leq 2\sint(\btX, \bU)$. By definition, there exits an orthogonal matrix $\bQ$ such that 
$$|| \bU \bQ - \btX||_2 \le 2 \sint(\btX, \bU) \le d.$$
Finally, since $\btX$ and $\bU\bQ$ are orthogonal, we know that 
$$ \btX^\top ( \bU\bQ - \btX) +  ( \bU\bQ - \btX)^\top \btX = - ( \bU\bQ - \btX)^\top ( \bU\bQ - \btX)$$
which implies  
$$|| \btX^\top ( \bU\bQ - \btX) +  ( \bU\bQ - \btX)^\top \btX  ||_2 = ||( \bU\bQ - \btX)^\top ( \bU\bQ - \btX)||_2 \le || ( \bU\bQ - \btX)^\top ( \bU\bQ - \btX)||^2_2 \le d^2.$$
This shows that the solution to our SDP exists. 

\paragraph{(Desired properties)}
Now we show that the randomly rounded solution has the required properties with high probability. 
We first prove some nice properties of $\bX$, and then use them to prove the properties of $\bbX$.

\begin{claim} \label{cla:bXproperty}
$\bX$ satisfies the following properties.\\
(a). Orthogonality property. 
$$\Pr \left[ ||\bX^{\top} \bX - \bI ||_2 \ge \frac{1}{4} \right] \le \frac{1}{8}.$$
(b). Spectral property. 
 $$ \Pr\left[\exists i\in [n], \leftll\bX^{\top}  \bD_i \bX -  \sum_{r = 1}^n \bW_{i, r}  \bA_r \rightll_2 \ge  \frac{\lambdal}{2} \right] \le \frac{1}{8}.$$
(c). Distance property. 
$$\Pr\left[ ||\btXb^{\top}  \bX||_2 \ge dk \sqrt{ \log n} \right] \le \frac{1}{8}.$$
(d). Incoherent property. 
$$ \forall r \in [n], (\bX^r) (\bX^r)^{\top}  \le  \frac{\mu k(k+1)}{n}.$$
\end{claim}

\begin{proof}[Proof of Claim~\ref{cla:bXproperty}]
It is easy to verify that in expectation, the rounded solution satisfies the properties stated:\\
(a). Orthogonality property. 
$$\E[\bX^{\top} \bX] =  \sum_{r = 1}^n\E[  (\bX^r)^{\top} \bX^r] = \sum_{r = 1}^n  \left( (\bR^r)^{\top} \bR^r + \bS_r \E[\sigma^\top \sigma] \bS_r \right) = \sum_{r = 1}^n \bA_r = \bI$$
since $\bX^r = \bR^r + \sigma \bS_r $ where $\bS_r$ is a PSD matrix with $ \bS_r^2 = \bA_r - (\bR^r)^{\top}(\bR^r)$. \\
(b). Spectral property. $$\E[\bX^{\top} \bD_i \bX] = \sum_{r = 1}^n\E[ \bW_{i, r} (\bX^r)^{\top} (\bX^r)] = \sum_{r = 1}^n \bW_{i, r}  \bA_r.$$
(c). Distance property. $$\E[\bX] = \bR, \quad \E[\btXb^{\top} \bX] = \btXb^{\top} \bR.$$ 
(d). Incoherent property. $$\E[(\bX^r) (\bX^r)^{\top}] = \trace(\bA_r).$$
Therefore, we just need to show that the random variables in  (a), (b), (c),  and (d) concentrate around their expectation. 

First consider (a). We can apply the matrix concentration lemma (\ref{lem:matrix_concentration}), for which we need to bound $||\sum_{r = 1}^n \bDelta_r ||_2$ where  $\bDelta_r = \bA_r - (\bR^r)^{\top} (\bR^r)$. Note that $\sum_{r = 1}^n \bA_r = \bI$, it suffices to bound $\sigma_{\min}\left(\sum_{r } (\bR^r)^{\top} (\bR^r)\right) = \sigma_{\min}\left(\bR^{\top} \bR \right)$.  Since $\bR  = \bR + \btX - \btX$, we have  
 $$\sigma_{\min}\left(\bR^{\top} \bR \right) = \sigma_{\min} \left(\btX^{\top} \btX + (\bR - \btX)^{\top}(\bR - \btX) + \btX^{\top} (\bR - \btX) + (\bR - \btX)^{\top} \btX \right).$$ 
 
Then by $||\bR - \btX||_2 \le d$, $||\btX^{\top}(\bR - \btX) + (\bR - \btX)^{\top} \btX||_2 \le d^2$ and $\btX^{\top} \btX = \bI$, we get
$$\sigma_{\min}\left(\bR^{\top} \bR \right) \ge 1  - ||(\bR - \btX)^{\top}(\bR - \btX)||_2 - ||\btX^{\top}(\bR - \btX) + (\bR - \btX)^{\top} \btX||_2 \geq 1 - 2 d^2.$$ 
Therefore, 
$$\leftll\sum_{r = 1}^n \bDelta_r \rightll_2 =  \leftll \sum_{r = 1}^n \bA_r  - \sum_{r=1}^n  (\bR^r)^{\top} (\bR^r) \rightll_2 =  \leftll\bI  - \bR^{\top} \bR \rightll_2 \le 2d^2. $$ 
Using the matrix concentration lemma (\ref{lem:matrix_concentration}) with $\Delta = 2d^2, \trace(\bDelta)_{\max} \le \frac{\mu k}{n}, (||a||_2^2)_{\max} = \frac{\mu k}{n} , t = 1/4$, we obtain that when $n$ is sufficiently large: 
$$\Pr \left[ \leftll \bX^{\top} \bX - \bI \rightll_2 \ge \frac{1}{4} \right] \le \frac{1}{8}.$$

Next consider (b). We can also apply the matrix concentration lemma (\ref{lem:matrix_concentration}) for each $i \in [n]$ and then take the union bound. Here, $\bDelta_r = \bW_{i, r} \left(\bA_r - (\bR^r)^{\top} (\bR^r) \right)$, so 
$$ \leftll \sum_{r = 1}^n \bDelta_r \rightll_2 \le  \leftll \sum_{r = 1}^n \bW_{i, r} \bA_r \rightll_2   \le \lambdau. $$ 
Using the matrix concentration lemma (\ref{lem:matrix_concentration}) with $\Delta = \lambdau, \trace(\bDelta)_{\max} = \frac{\mu k}{n} ||\bW||_{\infty}, (||a||_2^2)_{\max} = \frac{\mu k}{n}  ||\bW||_{\infty} , t = \frac{\lambdal}{2}$, we obtain that for any $i \in [n]$, when 
 $$\lambdal \ge  \sqrt{\frac{32 k^2 \mu || \bW||_{\infty} \log n}{n} \lambdau},$$
 we have
 $$ \Pr\left[\leftll\bX^{\top}  \bD_i \bX -  \sum_{r = 1}^n \bW_{i, r}  \bA_r \rightll_2 \ge  \frac{\lambdal}{2} \right] \le \frac{1}{8n}.$$
Taking the union bound leads to the desired property. 

Now consider (c). By triangle inequality, 
\begin{align*}
||\btXb^{\top}  \bX||_2 &\le ||\btXb^{\top} \bR||_2 + || \btXb^{\top} (\bX - \bR  )||_2.
\end{align*}
By the SDP, $||\btXb^{\top} \bR||_2 \le d$, so it suffices to bound $\btXb^{\top} (\bX - \bR  ) = \sum_{r = 1}^n ([\btXb]^r)^{\top}(\bX - \bR  )^r$. 
Let $\bZ_r =  ([\btXb]^r)^{\top}(\bX - \bR  )^r $, we have $\btXb^{\top} (\bX - \bR  )  = \sum_{r = 1}^n \bZ_r$.  Furthermore, $\E[\bZ_r] = 0$ with 
\begin{align*}
\leftll \E[ \sum_{r = 1}^n \bZ_r \bZ_r^{\top} ] \rightll_2 & \le \leftll \E  [ \sum_{r = 1}^n (\bX - \bR  )^r [(\bX - \bR  )^r]^{\top} ] \rightll_2 =  \sum_{r = 1}^n\trace( \bDelta_r ) \le 3d^2k,
\\
\leftll \E[ \sum_{r = 1}^n \bZ_r^{\top}  \bZ_r] \rightll_2 & \le \leftll \sum_{r = 1}^n \bDelta_r \rightll_2   \le 3d^2,
\\
||\bZ_r||_2 &\le 2dk.
\end{align*}  
By Matrix Bernstein inequality,  when $n$ is sufficiently large,  
$$\Pr\left[ || \btXb^{\top} (\bX - \bR  )||_2 \ge \frac{1}{2} dk \sqrt{ \log n} \right] \leq \frac{1}{8}.$$
The property then follows from the triangle inequality. 

Finally, consider (d). We know that $\bX^r = \bR^r + \sigma \bS_r $ where $\bS_r$ is a PSD matrix with $ \bS_r^2 = \bA_r - (\bR^r)^{\top}(\bR^r)$. Therefore, 
$$(\bX^r) (\bX^r)^{\top} = (\bR^r) (\bR^r)^{\top} + \sigma \bS_r^2 \sigma^{\top} \le \trace(\bA_r) + \trace(\bA_r) ||\sigma ||_2^2 \le \frac{\mu k(k+1)}{n}.$$
This completes the proof of the claim.
\end{proof}

We are now ready to prove the properties of $\bbX = \QR(\bX)$, the final output of $\AlgWhitening$. 
Assume none of the bad events in Claim~\ref{cla:bXproperty} happen.
First, by the spectral property (b) of $\bX$ in the claim, we have that for any $i \in [n]$,
$$\frac{\lambdal}{2} \bI  \preceq \bX^{\top}  \bD_i \bX   \preceq 2 \lambdau \bI.$$
Note that $||\bX^{\top} \bX - \bI || \le \frac{1}{4}$, which implies that $\sigma^2_{\max }(\bX) \le \frac{5}{4}$, $\sigma^2_{\min }(\bX) \ge \frac{3}{4}$. Therefore, for any $i \in [n]$,
$$ \bbX^{\top}  \bD_i \bbX  \succeq \frac{1}{\sigma^2_{\max }(\bX) } \bX^{\top}  \bD_i \bX \succeq \frac{\lambdal}{4} \bI $$
and 
$$ \bbX^{\top}  \bD_i \bbX  \preceq \frac{1}{\sigma^2_{\min }(\bX) } \bX^{\top}  \bD_i \bX \preceq 4\lambdal \bI.$$

Next, note that 
$$ \sint(\bbX, \btX) =  || \btXb^{\top} \bbX ||_2   \le \frac{1}{\sigma_{\min }(\bX) }  || \btXb^{\top} \bX ||_2 \le   \frac{3}{2} || \btXb^{\top} \bX ||_2  \le \frac{3}{2} dk \sqrt{\log n}.$$ 
Since $\sint(\btX, \bU) \le   \tant(\btX, \bU) \le \frac{d}{2}$, we have 
$$ \sint(\bbX, \bU) \le  \frac{d}{2} +  \frac{3}{2} dk \sqrt{\log n} \le 2  dk \sqrt{\log n} \le 1/2.$$
When $\text{sin} \theta \le 1/2$,  $\text{tan}\theta \le 2 \text{sin}\theta$. So 
$$\tant(\bbX, \bU) \le 2 \sint(\bbX, \bU) \le 4d k \sqrt{\log n}.$$
Finally, for incoherence, we know that $\incoherentpartial(\bX) \le \mu(k+1) $, $\sigma_{\min}(\bX) \ge \frac{3}{4}$. Then the output $\bbX$ satisfies $\incoherentpartial(\bbX) \le 4 \incoherentpartial(\bX) \le 5 \mu k$. 

$\bold{Note}$: since all the property of output $\bbX$ can be tested in polynomial time (for (3) we can test it using the input matrix $\btX$ because $\tant(\btX, \bU) \le \frac{d}{2}$), we can run the whitening algorithm for $O(\log(1/\alpha))$ times (using fresh randomness for the choice of $\bX$) and we will have success probability $1 - \alpha$. 
\end{proof}

\subsection{Final result}
\begin{thm} 
\label{thm:main} 
If $\bMg, \bW$ satisfy assumptions \textbf{(A1)-(A3)}, and
\begin{align*}
 ||\bW||_{\infty} & = O\left(\frac{\lambdal^2 n}{k^2 \mu \lambdau \log n} \right),  ~~
\gamma  = O\left( \frac{\lambdal \sigma_{\min}(\bMg)}{k^{3} \mu \sqrt{\log n}}\right),
\end{align*}
then after $O(\log(1/\epsilon))$ rounds Algorithm~\ref{alg:main_whitening} outputs a matrix $\btM$ that with probability $\ge 1-1/n$ satisfies
\begin{align*} 
  ||\btM - \bMg||_2 \le O\left(\frac{k^{3/2} \sqrt{\log n}}{\lambdal \sigma_{\min}(\bMg)}\right) ||\bW \hp \bNoise||_2 + \epsilon.
\end{align*}
The running time is polynomial in $n$ and $\log(1/\epsilon)$.
\end{thm}

The theorem is stated in its full generality. 
To emphasize the dependence on the matrix size $n$, the rank $k$ and the incoherency $\mu$, we  can consider a specific range of parameter values where the other parameters (the lower/upper spectral bound, the condition number of $\bMg$) are constants, which gives a corollary which is easier to parse. Also, these parameter values show that we can handle a wider range of parameters than the simple algorithm with the clipping as a whitening step.

\begin{cor}\label{cor:main_whitening} 
Suppose $\lambdal, \lambdau$ and  $\sigma_{\min}(\bMg)$ are all constants, and $T = O(\log(1/\epsilon))$. Furthermore, 
$$\|\bW\|_\infty = O\left(\frac{n}{k^2 \mu\log n} \right), ~~\gamma = O\left(\frac{1}{k^{3} \mu\sqrt{ \log n}} \right).$$ 
Then with probability $\ge 1-1/n$, 
$$||\btM - \bMg||_2 \le O\left(k^{3/2} \sqrt{\log n}\right) ||\bW \hp \bNoise||_2 + \epsilon.$$
\end{cor}


We now consider proving the theorem. 
After proving these lemmas, the proof is rather immediate.
Define the following two quantities:
$$\valupdate  = 4 k \sqrt{\log n},~~\valdelta =  \frac{64\sqrt{k} }{ \lambdal \sigma_{\min}(\bMg)}\valupdate= \frac{256 k^{3/2} \sqrt{\log n}}{\lambdal \sigma_{\min}(\bMg)}.$$
We just need to show that $\tant(\bbX_t, \bU)  \le \frac{1}{2^t} + \valdelta \delta$ for every $t\ge 1$,  and $\tant(\bbY_t, \bU) \le \frac{1}{2^t} + \valdelta \delta$ for every $t > 1$. We will prove it by induction. 

(a). After initialization, by Lemma \ref{lem:initialization} and \ref{lem:whitening}, we have
$$
\tant(\bbY_1, \bV) \le 4kd_1 \sqrt{\log n}  = d_1 \valupdate  = \frac{1}{2} + \valdelta \delta.
$$

(b). Suppose $\tant(\bbX_t, \bU)$ and $\tant(\bbY_t, \bV) \le \frac{1}{2^t} + \valdelta\delta$ is true for $t$, and consider the iterates at step $t + 1$. Since $\bbY_t$ is given by $\AlgWhitening$, by Lemma  \ref{lem:whitening}, we know that $\frac{1}{4} \lambdal \bI \preceq \bbY_t^{\top} \bD_i \bbY_t \preceq 4\lambdau \bI$ and $\bbY_t$ is $(5k\mu)$-incoherent. Therefore, applying Lemma \ref{lem:update} we have 
\begin{align*} 
\tant(\btX_{t+1}, \bU) & \le \frac{\tant(\bbY_{t}, \bV)}{4 \valupdate} + \frac{16 \sqrt{k} \delta}{ \lambdal \sigma_{\min}(\bMg)} \\
& \le \frac{1}{2^{t+ 2} \valupdate} + \left(\frac{\valdelta}{4 \valupdate} + \frac{16 \sqrt{k}}{ \lambdal \sigma_{\min}(\bMg)} \right) \delta \\
& \le   \frac{1}{2^{t+ 2} \valupdate}  +\frac{32 \sqrt{k} }{ \lambdal \sigma_{\min}(\bMg)}\delta.
\end{align*} 
Now, we know that $\tant(\btX_{t+1}, \bU) \le \frac{d_{t + 1}}{2} $ for $d_{t+1} = \frac{1}{2^{t+1} \valupdate} + \frac{64\sqrt{k}}{ \lambdal \sigma_{\min}(\bMg)} \delta$. By Lemma \ref{lem:whitening}, 
\begin{align*}
\tant(\bbX_{t + 1}, \bU) & \le d_{t+1} \valupdate \\
& \le \left( \frac{1}{2^{t+ 1} \valupdate} + \frac{64\sqrt{k}}{ \lambdal \sigma_{\min}(\bMg)}  \delta \right) \valupdate  \\
& \le  \frac{1}{2^{t + 1}} +  \valdelta \delta.
\end{align*}
Using exactly the same argument we can show that $\tant(\bbY_{t + 1}, \bV)  \le  \frac{1}{2^{t + 1}} +  \valdelta \delta$.

Then the theorem follows by bounding $\|\bMg - \btM\|_2$ by $\tant(\bbY_{T + 1}, \bV), \tant(\bbX_{T + 1}, \bU)$ using the triangle inequality and the spectral property of $\bW$. For simplicity, let $\bX = \bbX_{T+1}$ and $\bY = \bbY_{T+1}$. 

By definition, we know that there exists $\bQ_x $ and $\bQ_y$ such that $\bX \bQ_x = \bU + \Delta_x$  and $\bX \bQ_y = \bV + \Delta_y$ where $\| \Delta_x\|_2 = O(\tant(\bX, \bU))$ and $\| \Delta_y\|_2 = O(\tant(\bY, \bV)) $.

\begin{align*}
\| \bW \hp (\bM - \bX \bbSigma \bY^\top) \|_2 \leq \| \bW \hp (\bM - \bX \bQ_x \bSigma \bQ_y \bY^\top) \|_2 \\ 
\leq \| \bW \hp (\bMg + \bNoise - \bX \bQ_x \bSigma \bQ_y \bY^\top) \|_2 \\ 
\leq \| \bW \hp (\bMg + - \bX \bQ_x \bSigma \bQ_y \bY^\top) \|_2  + \| \bW \hp \bNoise \|_2 \\ 
\leq \| \bW \hp (\bMg + - \bX \bQ_x \bSigma \bQ_y \bY^\top) \|_2  + \| \bW \hp \bNoise \|_2.
\end{align*}
On the other hand,
\begin{align*}
\| \bW \hp (\bM - \bX \bbSigma \bY^\top) \|_2 \geq \| \bW \hp (\bMg - \bX \bbSigma \bY^\top) \|_2 - \| \bW \hp \bNoise \|_2.
\end{align*}
Therefore, 
\begin{align*}
\| \bW \hp (\bMg - \bX \bbSigma \bY^\top) \|_2 \leq  \| \bW \hp (\bMg - \bX \bQ_x \bSigma \bQ_y \bY^\top) \|_2  + 2 \| \bW \hp \bNoise \|_2 \\
 = O(\tant(\bX, \bU) + \tant(\bY, \bV)) + O(\| \bW \hp \bNoise \|_2).
\end{align*}

Define $\Delta = \bMg - \bX \bbSigma \bY^\top$ and $\Delta' = \bX \bQ_x  \bSigma \bQ_y^\top \bY^\top - \bX \bbSigma \bY^\top$, and note that the difference between the two is $O(\tant(\bX, \bU) + \tant(\bY, \bV)).$
\begin{align*}
\| \Delta\|_2  & \le \| \bW \hp \Delta\|_2 + \| (\bW - \bE) \hp \Delta\|_2  \\ 
& \le \| \bW \hp \Delta\|_2 + \| (\bW - \bE) \hp \Delta'\|_2 + O(\tant(\bX, \bU) + \tant(\bY, \bV)).
\end{align*}

So now it is sufficient to show that $\| (\bW - \bE) \hp \Delta'\|_2 \le c \| \Delta\|_2$ for a small $c < 1/2$.
Now we apply Lemma~\ref{lem:weighted_unweighted}. Let $\bZ = \bQ_x  \bSigma \bQ_y^\top  - \bbSigma$.
\begin{align*}
\| (\bW - \bE) \hp \Delta'\|_2 & = \| (\bW - \bE) \hp (\bX \bQ_x  \bSigma \bQ_y^\top \bY^\top - \bX \bbSigma \bY^\top) \|_2\\
& = \| (\bW - \bE) \hp (\bX \bZ \bY^\top) \|_2 \\
& \leq c \| \bZ \|_2 
\end{align*}
for some small $c < 1/2$, since $\gamma$ is small and $\bX$ and $\bY$ are incoherent. Note that $\bX$ and $\bY$ are projections, so $\| \bZ \|_2 = \| \bX \bZ \bY^\top\|_2 $, then
\begin{align*}
\| (\bW - \bE) \hp \Delta'\|_2 \leq c \| \Delta\|_2.
\end{align*}
Combining all things we have $\| \Delta\|_2  = O(\tant(\bX, \bU) + \tant(\bY, \bV)) + O(\| \bW \hp \bNoise \|_2) = O(\tant(\bX, \bU) + \tant(\bY, \bV))$, which completes the proof.

\section{Empirical verification of the spectral gap property} \label{app:experiment}

\begin{figure}[t]
\centering
\includegraphics[width=0.4\textwidth]{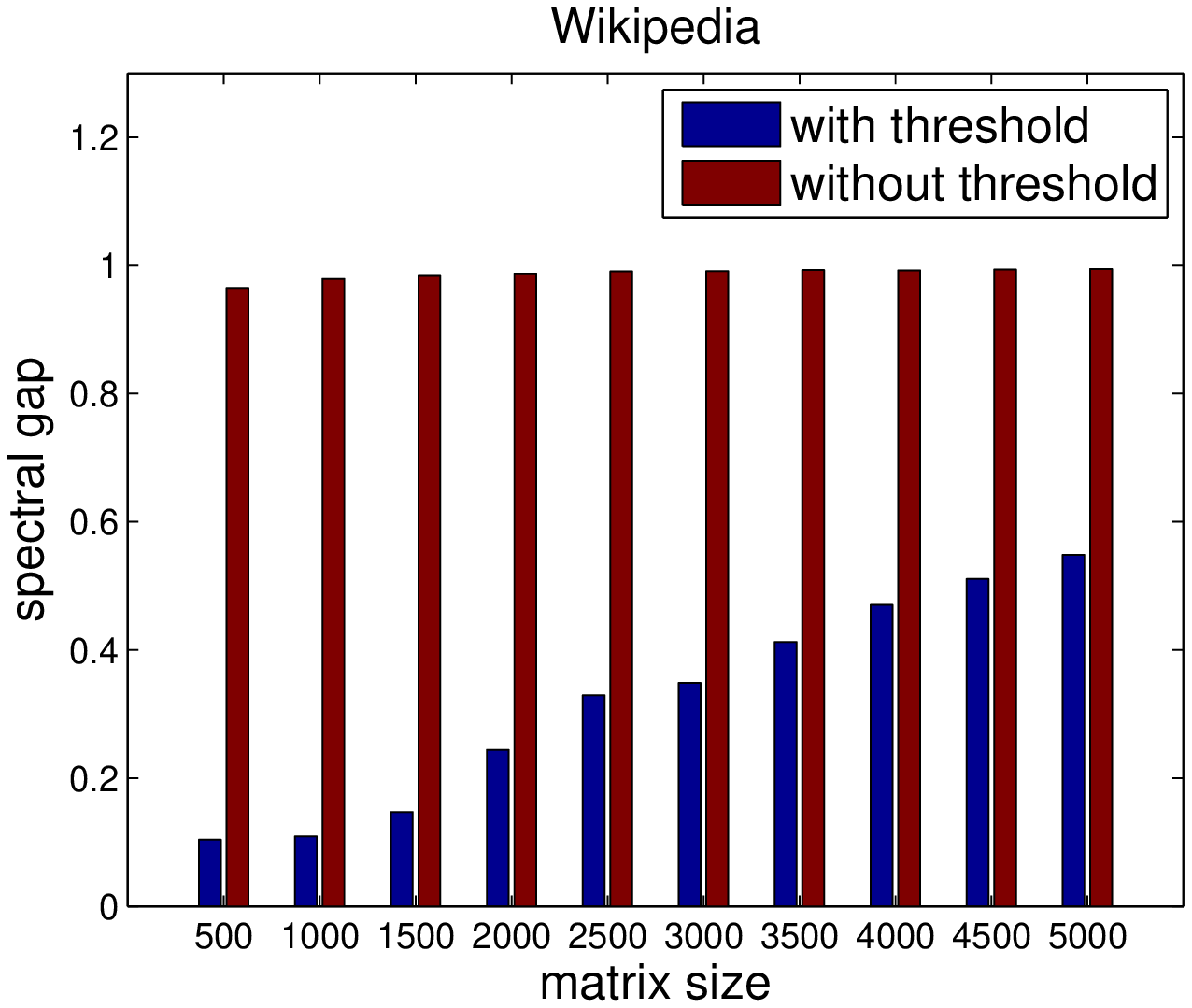}\includegraphics[width=0.4\textwidth]{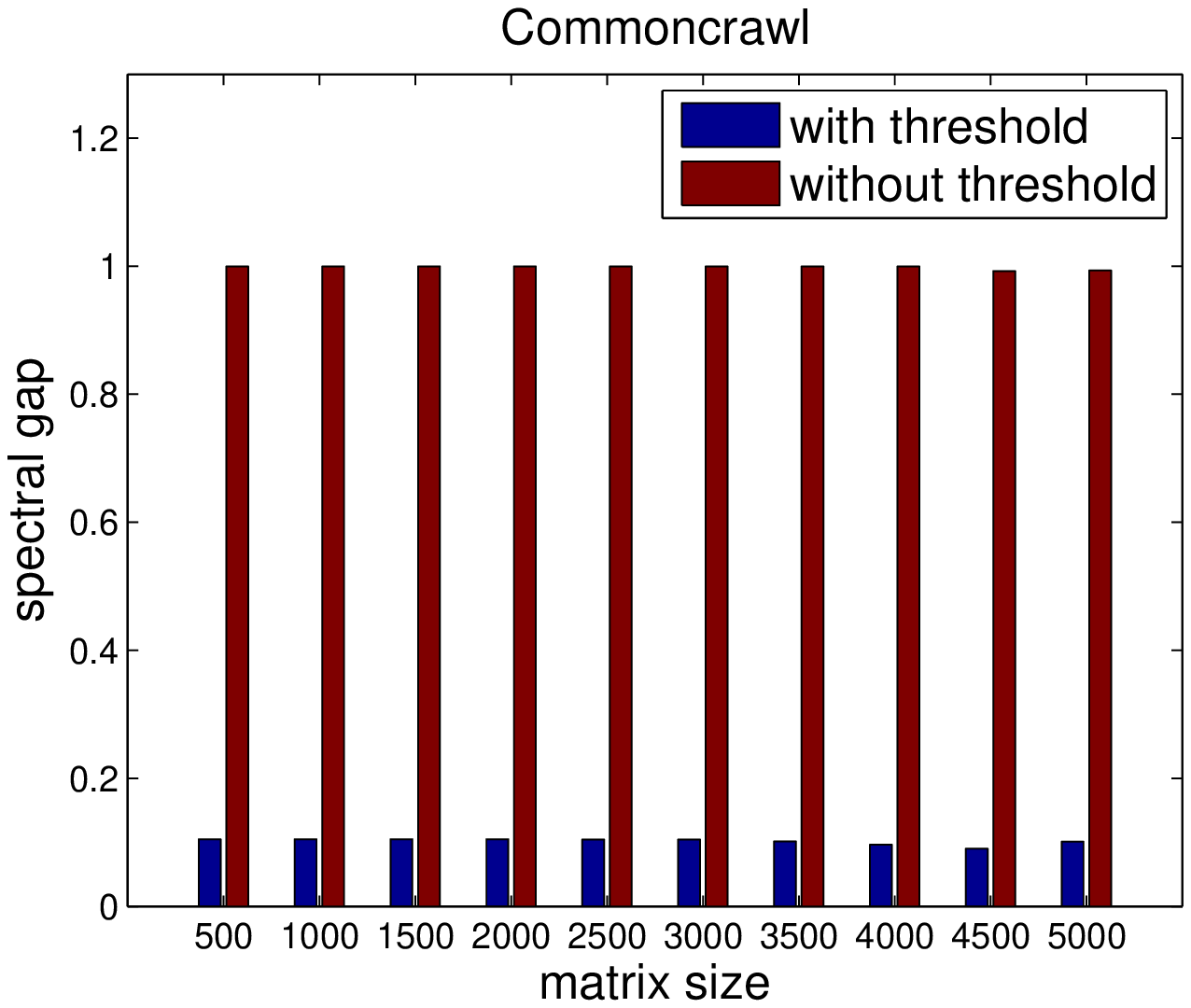}
\caption{Spectral gap of the weight matrix for word embeddings on two corpora. $x$-axis: number of words (size of the matrix); $y$-axis: the spectral gap $\|\bW - \bE\|_2$ where $\bE$ is the all-one matrix. } 
\label{fig:wlr}
\end{figure}

Experiments on the performance of the alternating minimization can be found in related work (e.g., \cite{lu1997weighted,srebro2003weighted}). Therefore, we focus on verifying the key assumption, i.e., the spectral gap property of the weight matrix (Assumption $\bold{(A2)}$). 

Here we consider the application of computing word embeddings by factorizing the co-occurrence matrix between the words, which is one of the state-of-the-art techniques for mapping words to low-dimensional vectors (about 300 dimension) in natural language processing. There are many variants (e.g., ~\cite{levy2014neural,pennington2014glove,arora2015random}); we consider the following simple approach. Let $X$ be the co-occurrence matrix, where $\bX_{i,j}$ is the number of times that word $i$ and word $j$ appear together within a window of small size (we use size $10$ here) in the given corpus.  Then the word embedding by weighted low rank problem is 
$$
  \min_{\bV} \sum_{i,j} f(\bX_{i,j}) \left(  \log\left(\frac{\bX_{i,j}}{X}\right) -  \langle \bV_i, \bV_j \rangle \right)^2 
$$
where $X = \sum_{i,j} \bX_{i,j}$, $\bV_i$'s are the vectors for the words, and $f(x) = \max\{\bX_{i,j}, 100\}$ for a large corpus and $f(x) = \max\{\bX_{i,j}, 10\}$ for a small corpus.

We focus on the weight matrix $\bW_{i,j} = f(\bX_{i,j})$. It has been observed that using $\bX_{i,j}$ as weights is roughly the maximum likelihood estimator under certain probabilistic model and is better than using uniform weights. It has also been verified that using the truncated  weight $f(\bX_{i,j})$ is better than using $\bX_{i,j}$. Our experiments suggest that $f(\bX_{i,j})$ is better partially due to the requirement that the weight matrix should have the spectral gap property for the algorithm to succeed.

We consider two large corpora (Wikipedia corpus~\cite{enwiki}, about 3G tokens; a subset of Commoncrawl corpus~\cite{Buck-commoncrawl}, about 20G tokens). For each corpus, we pick the top $n$ words ($n=500, 1000, \ldots, 5000$) and compute the spectral gap $\|\bW - \bE\|_2$ where $\bW$ is the weight matrix corresponding to the words, and $\bE$ is the all-one matrix. Note that a scaling of $\bW$ does not affect the problem, so we enumerate different scaling of $\bW$ (from $2^{-20}$ to $2^{10}$) and plot the best spectral gap.
We compare the two variants: with threshold ($\bW_{i,j} = f(\bX_{i,j})$), and without threshold ($\bW_{i,j} = \bX_{i,j}$).

The results are shown in Figure~\ref{fig:wlr}. Without threshold, there is almost no spectral gap. With threshold, there is a decent gap, though with the increase of the matrix size, the gap become smaller because larger vocabulary includes more uneven co-occurrence entries and thus more noise. This suggests that thresholding can make the weight matrix nicer for the algorithm, and thus leads to better performance.

\end{document}